\documentclass[journal]{IEEEtai}
\usepackage[hidelinks]{hyperref}

\usepackage{color}
\usepackage{hyperref}
\usepackage{url}
\usepackage{graphicx}
\usepackage{amsmath}
\usepackage{booktabs}
\usepackage{amssymb}
\usepackage{multirow}
\usepackage{array}
\usepackage{algorithmic}
\usepackage{algorithm}
\usepackage{bm}
\usepackage{xcolor}
\usepackage{amsthm}
\newtheorem{definition}{Definition}
\newtheorem{theorem}{Theorem}
\newtheorem{lemma}{Lemma}
\usepackage{graphicx}

\usepackage{amssymb}
\usepackage{booktabs}
\usepackage{multirow} 

\usepackage{amsmath}     %
\usepackage{amssymb}     %
\usepackage{bm}          %

\usepackage{times}
\usepackage{soul}
\usepackage{url}
\usepackage{graphicx}
\usepackage{amsmath}
\usepackage{amsthm}
\usepackage{booktabs}
\usepackage{algorithm}
\usepackage{algorithmic}
\usepackage[switch]{lineno}
\usepackage{multirow}

\usepackage{amsthm}
\newtheorem{proposition}{Proposition}
\newtheorem{corollary}{Corollary}

\newif\ifreview
\reviewfalse

\ifreview
  \newcommand{\revnote}[1]{\marginpar{\textcolor{red}{\scriptsize #1}}}
\else
  \newcommand{\revnote}[1]{}
  \definecolor{blue}{rgb}{0,0,0} %
\fi

\DeclareGraphicsExtensions{.pdf,.png,.jpg}

\definecolor{OliveGreen}{rgb}{0.33, 0.42, 0.18}

\ifCLASSINFOpdf
\else
\fi

\begin{document}
\title{Revisiting Cross-Attention Mechanisms: Leveraging Beneficial Noise for Domain-Adaptive Learning}

\author{Zelin~Zang, Yehui~Yang, Fei~Wang, Liangyu~Li, and Baigui~Sun%
	\thanks{Zelin Zang is with the School of Engineering, Westlake University, Hangzhou, China; the Centre for Artificial Intelligence and Robotics (CAIR), Hong Kong Institute of Science and Innovation, Chinese Academy of Sciences (HKISI-CAS), Hong Kong, China; the State Key Laboratory of Multimodal Artificial Intelligence Systems (MAIS), Institute of Automation, Chinese Academy of Sciences (CASIA), Beijing, China; and the School of Artificial Intelligence, University of Chinese Academy of Sciences (UCAS), Beijing, China (e-mail: zangzelin@westlake.edu.cn).}%
	\thanks{Yehui Yang and Liangyu Li are with the School of Engineering, Westlake University, Hangzhou, China.}%
	\thanks{Fei Wang and Baigui Sun are with Alibaba Group, Hangzhou, China.}%
	\thanks{Corresponding author: Zelin Zang.}}

\markboth{Journal of \LaTeX\ Class Files,~Vol.~14, No.~8, August~2024}%
{Zang \MakeLowercase{\textit{et al.}}: Revisiting Cross-Attention Mechanisms: Leveraging Beneficial Noise for Domain-Adaptive Learning}

\maketitle

\begin{abstract}
	Unsupervised Domain Adaptation (UDA) seeks to transfer knowledge from a labeled source domain to an unlabeled target domain but often suffers from severe domain and scale gaps that degrade performance. Existing cross-attention-based transformers can align features across domains, yet they struggle to preserve content semantics under large appearance and scale variations. \textcolor{blue}{To explicitly address these challenges, we introduce the concept of \textbf{beneficial noise}, which regularizes cross-attention by injecting controlled perturbations, encouraging the model to ignore style distractions and focus on content.}\revnote{R2-C2, C3} We propose the Domain-Adaptive Cross-Scale Matching (DACSM) \textcolor{blue}{framework}, which consists of a Domain-Adaptive Transformer (DAT) for disentangling domain-shared content from domain-specific style, and a Cross-Scale Matching (CSM) module that adaptively aligns features across multiple resolutions. \textcolor{blue}{DAT incorporates beneficial noise into cross-attention, enabling progressive domain translation with enhanced robustness, yielding content-consistent and style-invariant representations.} Meanwhile, CSM ensures semantic consistency under scale changes. Extensive experiments on VisDA-2017, Office-Home, and DomainNet demonstrate that DACSM achieves state-of-the-art performance, with up to +2.3\% improvement over CDTrans on VisDA-2017. \textcolor{blue}{Notably, DACSM achieves a +5.9\% gain on the challenging ``truck'' class of VisDA, evidencing the strength of beneficial noise in handling scale discrepancies.}\revnote{R2-C2} These results highlight the effectiveness of combining domain translation, \textcolor{blue}{beneficial-noise-enhanced attention}, and scale-aware alignment for robust cross-domain representation learning.
\end{abstract}

\begin{IEEEImpStatement}
	Many real-world AI applications suffer performance drops when the training data (source) and the deployment data (target) come from different distributions — a challenge known as domain shift. Our work introduces a new transformer-based framework that enhances domain adaptation by reinterpreting cross-attention as a domain translation mechanism. By injecting carefully designed “beneficial noise” into the attention process, the model learns to ignore irrelevant style features and focus on core semantic content, even under substantial visual or structural differences between domains. This enables robust adaptation in critical applications such as medical imaging, industrial inspection, and autonomous driving. Moreover, our model incorporates a novel cross-scale matching module to handle size discrepancies between training and deployment images. Experimental results on three widely-used benchmarks demonstrate consistent state-of-the-art performance. Overall, this work advances the frontier of robust, style-agnostic, and scalable domain-adaptive learning, helping to make AI models more resilient and trustworthy in practical, high-stakes environments.
\end{IEEEImpStatement}

\begin{IEEEkeywords}
	Unsupervised Domain Adaptation, Cross-Attention, Domain Adaptive Transformer, Cross-Scale Matching
\end{IEEEkeywords}

\IEEEpeerreviewmaketitle

\maketitle

\section{Introduction}
\label{sec:intro}
\begin{figure*}[t]
  \centering
  \includegraphics[width=\linewidth]{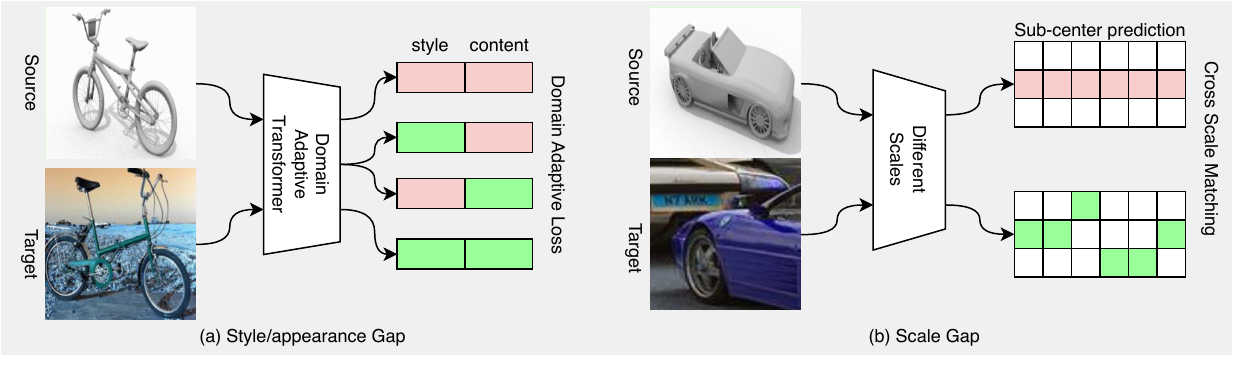}
  \caption{\textbf{Style/appearance gap and scale gap are two frequently observed challenges in UDA. }Two modules, Domain Adaptive Transformer and Cross Scale Matching, are carefully designed for the specific issues. (a), the two objects share the same structure but have different appearances. Additionally, identical semantic categories may manifest different scales in different datasets, which can also cause confusion for the model. (b) The scale gap is frequently observed in many datasets but has not received much attention. The objects may appear at different scales.}
  \label{fig:domain-gaps}
\end{figure*}

Deep learning has achieved remarkable success in computer vision tasks when abundant labeled data is available~\cite{xu2025adaptive}. However, domain shift remains a critical challenge when deploying models to new target domains with different data distributions~\cite{kang2025duda,yan2025source,ran2025camera}. {\textcolor{blue} Unsupervised Domain Adaptation (UDA) addresses this challenge by transferring knowledge from labeled source domains to unlabeled target domains~\cite{zang2023boosting,li2025cdcnet,ma2023unsupervised}, with most research focusing on learning domain-invariant features through distribution alignment~\cite{chhabra2024domain,ni2024tran}.}

Despite significant progress, existing UDA methods face two fundamental limitations. First, while appearance gaps—differences in visual style, texture, and rendering between domains—have been extensively studied, current approaches often fail to explicitly separate semantic content from domain-specific style information. Second, and more critically, the scale gap problem has been largely overlooked. As illustrated in Fig.~\ref{fig:domain-gaps}, identical semantic categories frequently appear at drastically different scales across datasets (e.g., partial objects versus complete scenes), causing severe feature misalignment that cannot be resolved by appearance-level adaptation alone. These two gaps, when unaddressed, fundamentally limit the effectiveness of cross-domain knowledge transfer.

The advent of Vision Transformers (ViT)\cite{tian2024fast} and cross-attention mechanisms\cite{chen2025datr,li2025unified} provides new opportunities for UDA. However, existing transformer-based methods like CDTrans~\cite{xu2023cdtrans} suffer from architectural limitations: their single-branch designs bias feature extraction toward source domains, and they fail to exploit the distinct semantic roles of different attention components. Through systematic analysis, we discover that transformer queries ($\bm{Q}$) primarily encode semantic content, while key-value pairs ($\bm{K}$/$\bm{V}$) capture style and appearance information \textcolor{blue}{\cite{deng2022stytr2,wu2021styleformer,park2019arbitrary,huang2017arbitrary}}.\revnote{R3-C1, R2-C1} This insight motivates our approach to explicitly disentangle content from style for more effective domain adaptation.

We propose a comprehensive framework that addresses both appearance and scale gaps through two complementary innovations. Our Domain Adaptive Transformer (DAT) employs a dual-branch architecture where domain-specific self-attention extracts features and cross-attention translates them across domains. {\textcolor{blue}By introducing beneficial noise into cross-attention and maintaining domain consistency in query generation, DAT promotes learning of content-biased, domain-invariant features while suppressing style-related variations.} To tackle the neglected scale gap, we design a Cross-Scale Matching (CSM) module with sub-center classifiers that learns semantically consistent representations across different object scales, enabling robust feature alignment regardless of scale variations between source and target domains.

{\textcolor{blue} Our method achieves state-of-the-art performance on challenging benchmarks including VisDA2017 and DomainNet, with significant improvements over existing approaches. The key contributions are threefold: (a) A Domain Adaptive Transformer with explicit content-style disentanglement through dual-branch cross-attention and beneficial noise injection; (b) A Cross-Scale Matching module that addresses the critical yet overlooked scale gap using adaptive sub-center classifiers; (c) Comprehensive experimental validation demonstrating superior performance through joint modeling of appearance and scale gaps in UDA.}

\section{Related Work}
\label{sec:related}
\subsection{Unsupervised Domain Adaptation}

In recent years, many methods have been proposed to solve unsupervised domain adaptation (UDA) problems where no labels are available for the target domain. They mainly focus on UDA based on domain alignment and discriminative domain-invariant feature learning methods. JAN~\cite{Ming2018style} designed a joint adaptation network to align the joint distributions of multiple domain-specific layers on the different domains by adopting a joint maximum mean discrepancy (JMMD) criterion. Maximum classification discrepancy (MCD)~\cite{Kuniaki2018} proposed to use the task-specific decision boundaries to align the distribution of source and target, which maximizes the discrepancy between two classifiers' outputs and minimizes this discrepancy of features.  A contrastive adaptation network (CAN)~\cite{Guoliang2018Contrastive} adopted to optimize the metric for minimizing the domain discrepancy, which explicitly models the intra-class compactness and the inter-class separability. Meanwhile Cross-Domain Gradient Discrepancy Minimization (CGDM)~\cite{Zhekai2018Zhekai} explicitly minimizes the discrepancy of gradients produced by source the samples and the target samples.
For instance, FixBi~\cite{Jaemin2021FixBi} creates bridges between the source and target domain by augmenting multiple intermediate domains, utilizing the intermediate domains. 
SHOT~\cite{Jian2020DoWR} align the feature representations between source and target domains by using both self-supervised pseudo-labeling and the information maximization.
Several recent works have proposed to use transformers to extract more transferable representations for domain adaptation. For example, TVT~\cite{Yang2021TVT} presented the Transferability Adaption Module(TAM) focusing on both transferable and discriminative features for domain adaptation. CDTrans takes advantage of pseudo-labeling to obtain paired samples and utilizes cross-attention to align features. DOT~\cite{DOT_2022_mm} obtains domain-invariant and domain-specific knowledge by using categorical tokens from both domains. PMTrans~\cite{PTrans2023} solves domain gap from a game theory's perspective, which bridges source and target domains with an intermediate domain.  Although these UDA methods can get the better performance. However, when significant domain gaps present, these methods can still make mistakes. 

\subsection{Attention Learning}
Attention learning has benefited from advances in deep learning. Attention mechanisms have become an important part of transduction models and compelling sequence modeling in various tasks, allowing modeling of dependencies regardless of their distance in the output or input sequences. Self-attention is an attention mechanism associated with different positions of a single sequence in order to compute features of the sequence. The cross-attention module is derived from the self-attention module, which aggregates similar patches of two input images. Some recent researchers use cross-attention methods for feature fusion, mainly in the multimodal tasks. In these works, the inputs of the cross-attention module ~\cite{Jichang2021Cross} are from two modalities, e.g. vision-to-text, text-to-speech, and vision-to-vision. They align and aggregate the information from two modalities by using the cross-attention module. Thanks to its powerful feature alignment, CDtrans adopted the cross-attention module to solve the UDA problem.

\textcolor{blue}{
\textbf{Novelty of Our Cross-Attention.} While cross-attention is a common building block, our utilization constitutes a paradigm shift in function, mechanics, and scope.
(1) \textbf{Functional Shift (Alignment $\to$ Translation):} Prior UDA works (e.g., CDTrans) use cross-attention discriminatively for feature alignment. In contrast, we use it generatively for \emph{Domain Translation}, synthesizing intermediate domain representations by reconstructing source content with target style statistics.
(2) \textbf{Mechanical Shift (Deterministic $\to$ Stochastic):} Standard cross-attention is deterministic. We introduce \emph{Beneficial Noise Injection}, transforming it into a stochastic process that prevents overfitting to spurious correlations.
(3) \textbf{Scope Shift (Fixed-Scale $\to$ Cross-Scale):} Unlike standard fixed-scale attention, we integrate it into a \emph{Cross-Scale Matching (CSM)} module to dynamically align features across resolutions.}\revnote{R2-C3}

{\color{blue} \textbf{Cross-Scale Transformers.} Beyond conventional attention designs, recent works focus on cross-scale attention architectures, such as Swin Transformer~\cite{liu2021swin}, CrossViT~\cite{chen2021crossvit}, and CSTrans\cite{liu2025cstrans}, primarily enhance representation capability by linking multi-resolution or multi-scale features. The objective of these methods is usually to improve performance within a single domain rather than addressing cross-domain adaptation. In contrast, our proposed \emph{Cross-Scale Matching (CSM)} is specifically designed for \emph{unsupervised domain adaptation} to address the scale discrepancy between source and target domains. Different from existing cross-scale transformer models, our approach (i) explicitly constructs cross-resolution correspondences by modeling scale-specific class centers through a sub-center classifier, and (ii) integrates with noise-gated cross-attention to enhance robustness under domain shifts. }
\subsection{Domain Translation} 
Due to recent development of GANs and availability of sufficient training data, domain translation~\cite{Jianxin2019,wang2022domain}  has also achieved impressive performance. 
Given an image in the source domain, the goal of domain translation is to learn the conditional distribution of corresponding images in the target domain. 
Some works enable the translation to preserve certain properties of the source domain data, such as pixel values, pixel gradients, semantic features, class labels, or pairwise sample distances. 
Another popular work is the cycle consistency loss, which enforces that if we translate an image to the target domain and back, we should obtain the original image. Domain translation approaches utilize the powerful ability of deep neural networks to build meaningful representations. 
For instance, Isola etal~\cite{Isola2017} designed a general conditional GAN framework for a wide range of supervised domain translation tasks. 
In our methods, we regard cross-attention as a domain translation module, which can learn domain-invariant representation for cross domain translation.

\subsection{Style Transfer and Image Editing}
Style transfer is an image synthesis task where the content of a source image is reproduced with the style of a target image. 
An early work, SplitMatch~\cite{frigo2016split}, uses a source image patch as the query to search for the best matched target patch. 
We elaborate on the details of StyleSwap~\cite{chen2016fast} to reveal the connection between cross attention and domain translation. Let $C$ and $S$ denote the RGB inputs of the content and style images respectively, and let $\Phi(\cdot)$ be the feature extractor of a pretrained CNN. The \emph{StyleSwap} procedure first extracts a set of feature patches for both content and style activations: $\{\Phi_i(C)\}_{i=1}^{n_c}$ and $\{\Phi_j(S)\}_{j=1}^{n_s}$, where $n_c$ and $n_s$ are the numbers of extracted patches. For each content patch $\Phi_i(C)$, it finds a closest matching style patch based on the inner product: %
\begin{equation}
  \text{ss}(i) = \mathop{\arg\max}_{j = 1, \dots, n_s} \langle\Phi_i(C), \Phi_j(S)\rangle.
  \label{eq:styleswap}
\end{equation}
Then it swaps each content activation patch with its closest-matching style patch: $\Phi_i(C) \leftarrow \Phi_{\text{ss}(i)}(S)$.
By reorganizing the patches or the feature vectors, their methods can produce good transferred results. \emph{The mechanism of matching}, such as Eq.~(\ref{eq:styleswap}), can be regarded as a hard cross attention where softmax is replaced by argmax. 
Following works like \cite{liu2021adaattn} apply the conventional attentional mechanism to arbitrary style transfer. 
Thus, the cross attention can be regarded as a domain translation module, where the query ($\bm{Q}_{\text{ue}}$) corresponds to the content and the key and value ($\bm{K}_{\text{ey}}$/$\bm{V}_{\text{al}}$) correspond to the style and appearance. 
However, the reorganization may break the original structure and lead to artifacts, and thus additional post-processing steps are required. 
Furthermore, RewritingGAN~\cite{bau2020rewriting} shows that it is possible to change the content of generated images via attentional mechanism. 
Note that, the cross attention or domain translation mechanism is \emph{not} the goal for UDA, while \emph{learning robust features (even if they are translated to the other domain)} is. Therefore, the best practice of cross attention for UDA is non-trivial, and we propose a feasible solution next. 

\section{Methodology}
\label{sec:method}
Unsupervised domain adaptation (UDA) aims to transfer knowledge from a labeled source domain \(\mathcal{X}^{s} = \{\bm{x}^{s}_i\}_{i=1}^{n^s}\), \(\mathcal{Y}^{s} = \{\bm{y}^{s}_i\}_{i=1}^{n^s}\), to an unlabeled target domain \(\mathcal{X}^{t} = \{\bm{x}^t_i\}_{i=1}^{n^t}\), under a shared classification task. The main challenge lies in bridging the distribution gap between domains, often caused by differences in style, texture, or scale.

We address this by decomposing features into two components: \emph{content}, which encodes domain-invariant semantics, and \emph{style}, which captures domain-specific variations. Empirical observations (Sec.~\ref{exper}) suggest that in Transformer-based models, query features (\(\bm{Q}\)) tend to represent content, while key and value features (\(\bm{K}, \bm{V}\)) are more related to style. Based on this, we propose a Domain-Adaptive and Cross-Scale Transformer (DACSM) that aligns features across domains using both self-attention and cross-attention. Self-attention encodes intra-domain information, while cross-attention enables feature translation by applying content queries to external styles. The model encourages classifiers to focus on query-driven content features and remain invariant to key/value-driven style.

To handle scale differences between domains, we introduce a cross-scale matching module. By augmenting source inputs at multiple resolutions and matching target features to the most compatible source scale, the model further improves alignment under scale shifts. DACSM unifies domain translation and scale adaptation in a single framework. Detailed components are described in the following sections.{\color{blue} An overview of the proposed Domain-Adaptive Cross-Scale Matching (DACSM) is shown in Fig.~\ref{fig:dacsm_frame}. }

\begin{figure}[t]
  \centering
  \includegraphics[width=0.9\linewidth]{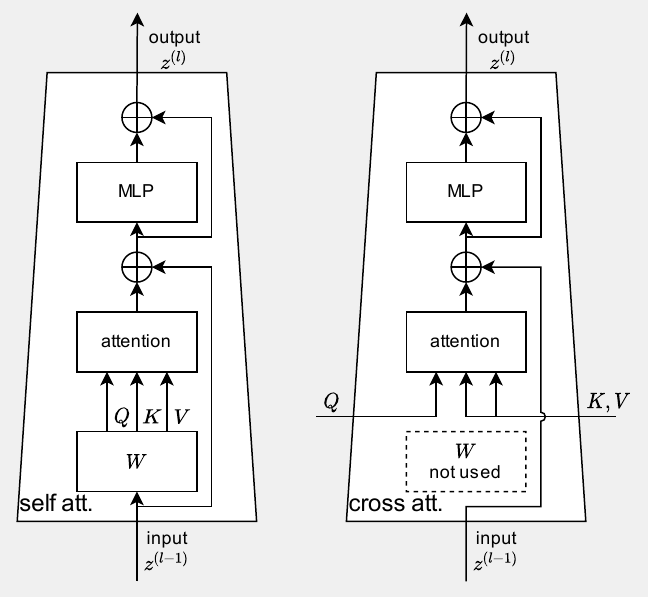}
  \caption{Transformer blocks with self and cross attentions. One attentional module consists of three input sources: the input feature \(\bm{z}\), the query \(\bm{Q}\), and the key and value pairs \(\bm{K}/\bm{V}\). In self attention, \(\bm{Q}\) and \(\bm{K}/\bm{V}\) come from the same domain. In cross attention, \(\bm{Q}\) and \(\bm{K}/\bm{V}\) may come from different domains.}
  \label{fig:self-cross-attn}
\end{figure}

\subsection{Preliminaries}

Vision Transformers (ViTs) represent an image \(\bm{x} \in \mathbb{R}^{H \times W \times C}\) as a sequence of non-overlapping patches. The image is divided into \(N = HW / P^2\) patches of size \(P \times P\), which are flattened and linearly projected to obtain patch embeddings \(\bm{x}_p \in \mathbb{R}^{N \times (P^2 C)}\). A learnable \([{\rm CLS}]\) token \(\bm{z}_0 \in \mathbb{R}^{D}\) is prepended to the sequence, and positional embeddings are added to retain spatial information.

The resulting input \(\bm{z} = [\bm{z}_0; \bm{z}_1; \dots; \bm{z}_N]\) is processed by \(L\) Transformer layers, each consisting of multi-head self-attention (MSA) and a feed-forward network (MLP). {\color{blue}Note that the tokenization (patch embedding) is performed only once at the beginning, and subsequent layers operate on these feature tokens.}\revnote{R1-C4} At layer \(l\), the update rule is:
\begin{equation}
  \begin{aligned}
    \bm{z}^{(l)}_{\text{msa}} & = \bm{z}^{(l-1)} + \text{MSA}(\bm{z}^{(l-1)}),                       \\
    \bm{z}^{(l)}              & = \bm{z}^{(l)}_{\text{msa}} + \text{MLP}(\bm{z}^{(l)}_{\text{msa}}),
  \end{aligned}
\end{equation}
where residual connections are applied after each sub-layer.

In the attention module, each token is projected into a query, key, and value:
\begin{equation}
  \bm{Q} = \bm{z} W_Q, \quad
  \bm{K} = \bm{z} W_K, \quad
  \bm{V} = \bm{z} W_V.
\end{equation}
The attention output is computed as:
\begin{equation}
  \text{A}(\bm{Q}, \bm{K}, \bm{V})
  = \text{Softmax}\left(\frac{\bm{Q} \bm{K}^{\top}}{\sqrt{D}}\right)\bm{V}.
  \label{eqn:attn}
\end{equation}

In standard self-attention, all three components are derived from the same input. In our UDA setting, the query, key, and value may come from different domains. When the query is from one domain (e.g., source) and the key/value from another (e.g., target), we refer to the attention as \emph{cross-attention}. Specifically, we denote
\begin{equation}
  \text{A}(\bm{Q}^s, \bm{K}^t, \bm{V}^t), \quad
  \text{A}(\bm{Q}^t, \bm{K}^s, \bm{V}^s)
\end{equation}
as cross-attention operations from source to target and vice versa.

The \([{\rm CLS}]\) token output from the final layer, \(\bm{z}^{(L)}_0\), is used as the class-level representation and denoted by \(\bm{f}\).

\begin{figure*}[t]
  \centering
  \includegraphics[width=0.99\linewidth]{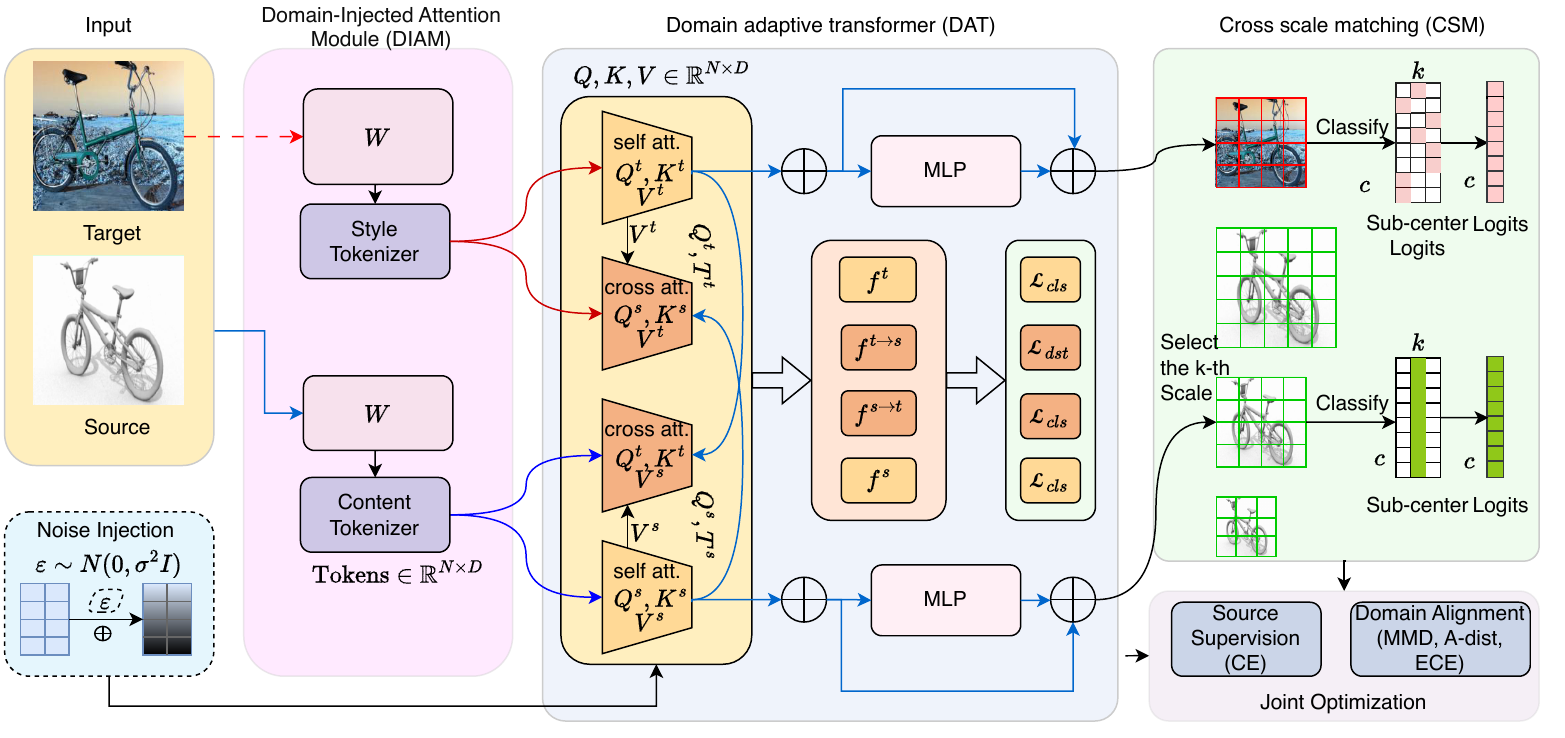}
  \caption{
    \textbf{DACSM Framework.} Domain-Adaptive Cross-Scale Matching (DACSM) integrates the Domain-Adaptive Transformer (DAT) and the Cross-Scale Matching (CSM) modules into a single end-to-end architecture.
    {\color{blue} Note that the Tokenizer (Patch Embedding) is applied only once at the input stage; subsequent blocks operate on feature tokens.}
  The DAT employs self- and cross-attention to disentangle domain-shared content (queries) from domain-specific style (keys/values), ensuring query-consistent and content-biased feature learning.
  {\color{blue} The residual connection adds the Source Query stream to the attention output, ensuring dimension matching.}
  Beneficial noise is injected into cross-attention to suppress spurious style correlations and enhance robustness. The CSM module addresses scale discrepancies by rescaling source images with multiple factors and aligning target features through a sub-center classifier. Each sub-center learns scale-specific class features in the source domain, while the target features are adaptively matched to the most compatible scale.}
  \label{fig:dacsm_frame} 
\end{figure*}

\subsection{Domain Translation via Cross Attention}

We reinterpret the attention mechanism in Transformers as a tool for domain translation. Specifically, we use cross-attention to map content representations from one domain into the style space of the other.
Given a query from the source domain \(\bm{Q}^s\) and key-value pairs from the target domain \((\bm{K}^t, \bm{V}^t)\), the attention operation
$\text{A}(\bm{Q}^s, \bm{K}^t, \bm{V}^t)$
computes a similarity-weighted combination of target values based on query-key matching. This allows source features to absorb style information from the target, producing a translated representation denoted as \(\bm{f}^{s \rightarrow t}\).
Symmetrically, applying
$
  \text{A}(\bm{Q}^t, \bm{K}^s, \bm{V}^s)
$
translates target queries using source keys and values, yielding \(\bm{f}^{t \rightarrow s}\).

These translated features enable the model to align content across domains under different styles. For comparison, we denote the features computed via self-attention within each domain as \(\bm{f}^s\) and \(\bm{f}^t\), respectively.

{\color{blue}
\textbf{Guarantees for Content-Style Decomposition.} Our decomposition is enforced by:
(1) \textbf{Architectural Asymmetry:} The Query comes from the Source stream and is preserved via residual connections, forcing it to encode structural content. Keys/Values from the Target stream only provide weighted statistics, naturally serving as style providers.
(2) \textbf{Statistical Constraints:} We employ orthogonal losses—Style Loss aligns K/V statistics (style), while Contrastive Loss enforces Q invariance (content). Beneficial Noise further regularizes K/V to prevent overfitting to specific style details.
(3) \textbf{Empirical Validation:} Reconstruction experiments (Fig.~\ref{fig:scross_construct}) confirm that combining Source Query with Target K/V yields source-content/target-style images.}\revnote{R2-C1}

\begin{figure}[t]
  \centering
  \includegraphics[width=0.9\linewidth]{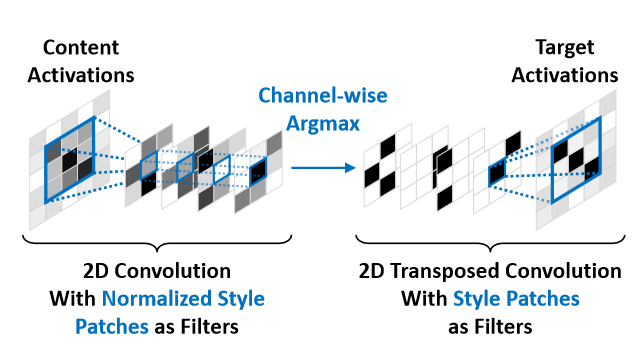}
  \includegraphics[width=0.9\linewidth]{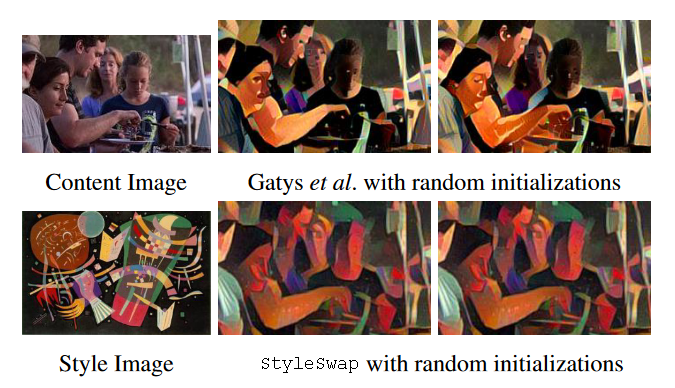}
  \caption{(Top) Illustration of a style swap operation. A 2D convolution extracts $3\times3$ patches with stride 1 and computes the normalized cross-correlations. There are $n_c=9$ spatial locations and $n_s=4$ feature channels immediately before and after the channel-wise argmax operation. A 2D transposed convolution reconstructs the full activations by placing each best-matching style patch at the corresponding spatial location. (Bottom) Style transfer results of StyleSwap and optimization-based methods. Both figures are borrowed from StyleSwap~\cite{chen2016fast}.}
  \label{fig:styleswap}
\end{figure}

We view the attention mechanism (Fig.~\ref{fig:self-cross-attn}) as a domain translation operator. Consider cross-attention
\(\text{A}(\bm{Q}^s, \bm{K}^t, \bm{V}^t)\), where a source-domain query is matched against target-domain keys. The similarity is computed via scaled dot-product:
\begin{equation}
  \text{Softmax}\left(
  \frac{\bm{Q}^s {\bm{K}^t}^\top}{\sqrt{D}}
  \right),
\end{equation}
which assigns higher weights to more similar key-query pairs. The output is a weighted combination of target values, effectively injecting target-style features into source queries.

By stacking multiple cross-attention layers, the source representation is progressively aligned to the target domain, yielding the translated feature \(\bm{f}^{s \rightarrow t}\). Similarly, cross-attention in the reverse direction produces \(\bm{f}^{t \rightarrow s}\), where target queries absorb source-domain characteristics.

For reference, features obtained via self-attention alone are denoted as \(\bm{f}^s\) and \(\bm{f}^t\) for source and target domains, {\color{blue} respectively (Fig.~\ref{fig:dacsm_frame} DAT module).}

\subsection{Beneficial Noise in Cross-Attention}
{\color{blue}
A key novelty of our framework is the explicit injection of \emph{beneficial noise} into the cross-attention mechanism. The motivation is to alleviate the model's over-dependence on domain-specific style features while retaining domain-shared content representations, thereby enhancing robustness and generalization under domain shifts. }

Let $Q \in \mathbb{R}^{n_q \times d}$ denote queries that encode content, and $K,V \in \mathbb{R}^{n_k \times d}$ denote cross-domain keys and values that encode style. We inject zero-mean Gaussian perturbations into $K$ and $V$ while keeping $Q$ intact:
\begin{equation}
\tilde{K} = K + \varepsilon_K,\quad \tilde{V} = V + \varepsilon_V,
\end{equation}
{\color{blue}where $\varepsilon_K, \varepsilon_V \in \mathbb{R}^{n_k \times d}$ are noise matrices with the same dimensions as $K$ and $V$. Each element in these matrices is independently sampled from a Gaussian distribution $\mathcal{N}(0, \sigma^2)$.}\revnote{R1-C1}

 then defined as:
\begin{equation}
\mathrm{Attn}(Q,\tilde{K},\tilde{V}) = 
\mathrm{softmax}\!\Big(\frac{Q\tilde{K}^\top}{\sqrt{d}}\Big)\tilde{V}.
\end{equation}

By preserving $Q$, we ensure content-consistency and maintain structural semantics, the $K$ and $V$, we weaken spurious correlations tied to domain-specific styles, forcing the model to rely more on content features. Noise injection acts as a regularizer, smoothing gradients, preventing overconfidence on noisy style tokens, and improving cross-domain generalization.

\subsection{Learning Domain Adaptive Transformer}

Cross-attention in Transformer blocks provides a natural mechanism for domain translation. Since the attention output is added as a residual to the input {\color{blue}(specifically, the Source Query stream is added to the attention output)}\revnote{R1-C3}, it enables gradual feature adaptation while preserving the original content. We identify two key principles for effectively learning domain-invariant representations.

\textbf{(1) Query consistency.} The query \(\bm{Q}\) should originate from the same domain as the input feature \(\bm{z}\). Although Transformers are permutation-invariant in theory, ViT patch embeddings preserve spatial structure. Mixing queries and residuals from different domains may disrupt this structure. Therefore, the query and input must be aligned to ensure meaningful residual learning.

\textbf{(2) Query encodes content.} Prior work~\cite{geirhos2018imagenet} highlights that shape-based features offer better robustness than texture-based ones. In ViTs, queries often capture higher-level content (e.g., shape), while keys and values encode appearance or style. To leverage this inductive bias, we encourage the model to derive class-discriminative signals from the query, facilitating more robust domain adaptation.
Based on these insights, we construct a \emph{Domain Adaptive Transformer (DAT)} framework (Fig.~\ref{fig:dacsm_frame}). Four types of class features are extracted using a parameter-shared Transformer under different attention configurations,
\(\bm{f}^s\) is the self-attention on the source input,
\(\bm{f}^{s \rightarrow t}\) is the source query with target key/value (cross-attention),
\(\bm{f}^{t \rightarrow s}\) is the target query with source key/value (cross-attention),
\(\bm{f}^t\) is the self-attention on the target input.
A shared classifier \(g\) maps each feature to class logits \(\bm{p} = g(\bm{f}) \in \mathbb{R}^C\), where \(C\) is the number of classes. We use \(g_c(\bm{f})\) to denote the logit corresponding to class \(c\).

\paragraph{Classification, Distillation, and Style Losses.}
We optimize the model using three types of loss functions: classification loss, distillation loss, and style perceptual loss.

\textbf{Classification losses.}
For labeled source data, we apply cross-entropy loss to both self-attention features \(\bm{f}^s\) and cross-attention features \(\bm{f}^{s \rightarrow t}\), both of which share the same source label:
\begin{equation}
  \begin{aligned}
    \mathcal{L}_{cls}^{s}               & = \mathbb{E}_{\bm{x}^s \in \mathcal{X}^s} \mathcal{L}_{CE}(g(\bm{f}^s), \bm{y}^s),                 \\
    \mathcal{L}_{cls}^{s \rightarrow t} & = \mathbb{E}_{\bm{x}^s \in \mathcal{X}^s} \mathcal{L}_{CE}(g(\bm{f}^{s \rightarrow t}), \bm{y}^s),
  \end{aligned}
  \label{eqn:da-source}
\end{equation}
where the cross-entropy loss is defined as
\begin{equation}
  \mathcal{L}_{CE}(\bm{p}, \bm{y}) = -\sum_{j=1}^C y_j \log \left( \frac{\exp(p_j)}{\sum_{k=1}^C \exp(p_k)} \right),
\end{equation}
and \(\bm{p}\) and \(\bm{y}\) denote the predicted logits and the one-hot ground-truth label, respectively.

\textbf{Distillation loss.}
For the target domain, we leverage \(\bm{f}^{t \rightarrow s}\) (a target feature translated into source style) as a teacher to guide \(\bm{f}^t\) via Kullback-Leibler divergence:
\begin{equation}
  \mathcal{L}_{dst} = \mathbb{E}_{\bm{x}^t \in \mathcal{X}^t}
  \mathcal{L}_{KL} \left(
  \sigma\left( \tfrac{g(\bm{f}^t)}{\tau} \right),
  \sigma\left( \tfrac{g(\bm{f}^{t \rightarrow s})}{\tau} \right)
  \right),
  \label{eqn:da-target}
\end{equation}
where \(\tau > 0\) is the temperature and \(\sigma(\cdot)\) denotes the softmax function.

\textbf{Pseudo-label loss.}
We optionally apply pseudo-label supervision (e.g., using SHOT~\cite{liang2020we} or CDTrans~\cite{xu2023cdtrans}) to further train \(\bm{f}^t\):
\begin{equation}
  \mathcal{L}_{cls}^t = \mathbb{E}_{\bm{x}^t \in \mathcal{X}^t}
  \mathcal{L}_{CE}(g(\bm{f}^t), \tilde{\bm{y}}^t),
  \label{eqn:da-pseudo}
\end{equation}
where \(\tilde{\bm{y}}^t\) is the one-hot pseudo label predicted by a teacher model.

\textbf{Style perceptual loss.}
To reduce domain-specific appearance differences, we align the feature statistics of source and target domains across layers:
\begin{equation}
  \mathcal{L}_s = \frac{1}{N} \sum_{i=1}^{N} \left(
  \left\| \mu(\bm{f}^s_i) - \mu(\bm{f}^t_i) \right\|_2 +
  \left\| \delta(\bm{f}^s_i) - \delta(\bm{f}^t_i) \right\|_2
  \right),
  \label{eqn:style}
\end{equation}
where \(\mu(\cdot)\) and \(\delta(\cdot)\) denote the mean and standard deviation across channels, and \(N\) is the number of Transformer layers.
Unlike contrastive objectives that may suffer from unstable gradients when the similarity distribution is skewed (e.g., softmax-based schemes in t-SNE), our method adopts a symmetric structure-preserving loss (SMM) whose gradient stability is theoretically guaranteed under mild assumptions (see Sec.~\ref{app:cross-style}). This ensures convergence efficiency and robustness during domain alignment.

\textbf{Overall objective.}
The total training loss is the weighted sum of all components:
\begin{equation}
  \mathcal{L}_{DA} =
  \mathcal{L}_{cls}^{s} +
  \mathcal{L}_{cls}^{s \rightarrow t} +
  \mathcal{L}_{dst} +
  \mathcal{L}_{cls}^{t} +
  \mathcal{L}_s.
  \label{eqn:da-all}
\end{equation}

Each term contributes to learning robust and domain-invariant representations:
\(\mathcal{L}_{cls}^{s}\) and \(\mathcal{L}_{cls}^{s \rightarrow t}\) supervise content-based learning from the source,
\(\mathcal{L}_{dst}\) transfers knowledge from translated target features,
\(\mathcal{L}_{cls}^t\) reinforces target prediction with pseudo labels, and
\(\mathcal{L}_s\) encourages statistical alignment across domains.

Additionally, the structure-preserving formulation of SMM allows us to maintain global inter-class relationships even under complex domain shifts. As we prove in Sec.~\ref{app:query-content}, the expected inter-class distance in the latent space remains close to that in the original space, avoiding global structure collapse.

The SMM loss plays a crucial role here: by minimizing intra-class dispersion while discouraging inter-class confusion, it naturally induces well-separated decision boundaries (see Sec.~\ref{app:csm-proof}), which complements the inductive bias of query-driven content modeling in Transformers.

\begin{algorithm}[t]
    \small
    \caption{Domain-Adaptive Cross-Scale Matching Module (DACSM).}
    \label{algorithm}
    \begin{algorithmic}[1]
        \REQUIRE {Source samples $\mathcal{X}^{s}$, target samples  $\mathcal{X}^{t}$, mini-batch B, warm-up epochs, the $\beta$ value, epochs $\mathcal{K}$}
        \ENSURE {Target classification predict $\mathcal{Y}^{t}$}
        \vspace{2mm}

        \FOR{ $k \leq \mathcal{K}$}
        \IF {$k\%n == 0$}
        \STATE Obtain initial pseudo-labels $\tilde{\bm{y}}^t$ by using $\mathop{min}$ $d(\boldsymbol{f}^s,\boldsymbol{f}^t)$.\\
        \STATE Obtain the training paired samples $(\mathcal{X}^{s},\mathcal{X}^{t})$. \\
        \ENDIF
        \STATE Capture multi-scale patch embeddings using CSM.  \\
        \STATE Train the DAT with $(\mathcal{X}^{s},\mathcal{X}^{t})$. \\
        \STATE Obtain classification predict $\mathcal{Y}^{t}$. \\
        \STATE Calculate classification loss; $\mathcal{L}_\text{cls}^{s}$ and $\mathcal{L}_\text{cls}^{s\rightarrow{t}}$ by Eq.~(\ref{eqn:da-source}).
        \STATE Calculate distillation loss; $\mathcal{L}_\text{dst}$ by Eq.~(\ref{eqn:da-target}).
        \STATE Calculate classification loss; $\mathcal{L}_\text{cls}^{t}$ by Eq.~(\ref{eqn:da-pseudo}).
        \STATE Calculate style perceptual loss; $\mathcal{L}_\text{s}$ by Eq.~(\ref{eqn:style}).
        \STATE Calculate total loss function $\mathcal{L}_{DA} = \mathcal{L}_{cls}^{s} + \mathcal{L}_{cls}^{s\rightarrow{t}} + \mathcal{L}_{dst} + \mathcal{L}_{cls}^{t} + \mathcal{L}_\text{s}$.\\
        \STATE Back-propagate with $\mathcal{L}_{DA}$.
        \STATE Update network parameters $\theta$. \\
        \ENDFOR
    \end{algorithmic}
    \label{alg_DACSM}
\end{algorithm}

\subsection{Detailed Analysis of  Cross-Scale Matching Module}

While DAT addresses domain gaps in style and appearance, it does not fully account for scale discrepancies, which affect shape and content representations encoded by queries (\(\bm{Q}\)).

Significant scale differences between source and target images  (Fig.~\ref{fig:dacsm_frame} CSM module) can lead to misaligned features and degraded class predictions. To mitigate this, we introduce a Cross-Scale Matching (CSM) module. Specifically, each source image \(\bm{x}^s\) is resized using \(K\) scale factors \(\{\gamma_1, \dots, \gamma_K\}\), producing augmented inputs \(r(\bm{x}^s, \gamma_k)\). These help the model cover a broader range of target-domain scales.

Each rescaled source input is processed by a shared Transformer, generating features \(\bm{f}^{s,k}\). Cross-attended counterparts \(\bm{f}^{s \rightarrow t, k}\) are obtained by applying cross-attention with the target domain. A sub-center classifier \(g^{csm}(\bm{f}, k') \in \mathbb{R}^C\) predicts class logits across \(K\) scale-specific centers. Let \(g^{csm}_c(\bm{f}, k')\) denote the logit for class \(c\) at scale \(k'\).

For source-domain inputs, we supervise each branch using its corresponding scale:
\begin{equation}
  \begin{aligned}
    g_c(\bm{f}^{s,k})                & = g^{csm}_c(\bm{f}^{s,k},\,k),                \\
    g_c(\bm{f}^{s \rightarrow t, k}) & = g^{csm}_c(\bm{f}^{s \rightarrow t, k},\,k).
  \end{aligned}
  \label{eqn:csm-source}
\end{equation}

For the target features \(\bm{f}^t\) and \(\bm{f}^{t \rightarrow s}\), the optimal scale is unknown. We select, for each class, the most confident sub-center prediction:
\begin{equation}
  \begin{aligned}
    g_c(\bm{f}^t)                 & = \max_{k'}\, g^{csm}_c(\bm{f}^t, k'),                 \\
    g_c(\bm{f}^{t \rightarrow s}) & = \max_{k'}\, g^{csm}_c(\bm{f}^{t \rightarrow s}, k').
  \end{aligned}
  \label{eqn:csm-target}
\end{equation}

CSM integrates into DAT by replacing the original classifier outputs with those defined in Eq.~\eqref{eqn:csm-source} and Eq.~\eqref{eqn:csm-target} in the overall loss \eqref{eqn:da-all}. Importantly, only the target input at the original resolution is used during inference, ensuring runtime efficiency. In addition, all branches share a single Transformer backbone, keeping memory usage low.

In summary, CSM complements DAT by addressing the scale gap. By learning scale-aware class features and aligning target features to the most compatible source scale, DACSM achieves improved performance under scale variation. The full training procedure is outlined in Algorithm~\ref{alg_DACSM}.

\subsection{Complexity Analysis}
{\color{blue}We analyze the computational complexity of DACSM. Let $N$ be the number of tokens, $D$ be the feature dimension, $L$ be the number of layers, and $K$ be the number of scales in CSM.
\begin{itemize}
    \item \textbf{Training:} The baseline complexity is $\mathcal{O}(L \cdot N^2 D)$. CSM introduces cross-attention across $K$ scales, adding $\mathcal{O}(K \cdot N^2 D)$. Since $K$ is small (e.g., 3), the total complexity $\mathcal{O}((L+K) \cdot N^2 D)$ remains quadratic and feasible.
    \item \textbf{Inference:} Crucially, CSM and the target branch are used only during training. During inference, we use only the trained backbone at the original resolution. Thus, inference complexity is $\mathcal{O}(L \cdot N^2 D)$, identical to the baseline, ensuring zero additional latency.
\end{itemize}}\revnote{R3-C2}

\subsection{Detailed Analysis of Relationship to Style Transfer and Image Editing}

Style transfer aims to synthesize an image that preserves the content of a source while adopting the style of a target. Early methods, such as SplitMatch~\cite{frigo2016split}, match patches from the content image to the closest style patches for replacement.

We highlight the connection between cross attention and style transfer via StyleSwap~\cite{chen2016fast}. Let \(C\) and \(S\) denote the content and style images, and let \(\Phi(\cdot)\) be a pretrained CNN feature extractor. StyleSwap proceeds as follows:
\begin{enumerate}
  \item Extract local feature patches \(\{\Phi_i(C)\}_{i=1}^{n_c}\) and \(\{\Phi_j(S)\}_{j=1}^{n_s}\) from the activations of \(C\) and \(S\).
  \item For each content patch \(\Phi_i(C)\), find the most similar style patch via normalized inner product:
        \begin{equation}
          \mathrm{ss}(i) = \arg\max_{j} \frac{\langle \Phi_i(C), \Phi_j(S) \rangle}{\|\Phi_i(C)\| \cdot \|\Phi_j(S)\|}.
          \label{eq:styleswap}
        \end{equation}
  \item Replace \(\Phi_i(C)\) with its best-matching style patch \(\Phi_{\mathrm{ss}(i)}(S)\).
  \item Reconstruct the final output \(\Phi^{ss}(C, S)\) by averaging overlapping regions.
\end{enumerate}

This process produces features that retain the structure of \(C\) and the texture of \(S\) (Fig.~\ref{fig:styleswap}). The matching operation in Eq.~\eqref{eq:styleswap} can be viewed as a form of hard cross attention—replacing the softmax with an argmax.

Later works~\cite{park2019arbitrary,liu2021adaattn} extend this idea using soft attention to support arbitrary style transfer. Under this view, cross attention can be interpreted as a domain translation operator: the query \(\bm{Q}\) captures content, while the key-value pair \((\bm{K}, \bm{V})\) encodes style.
Unlike traditional style transfer methods that modify pixel-level features and may introduce artifacts, our goal in UDA is not synthesis but learning robust, domain-invariant representations.

\section{Experiments}
\label{sec:experiments}
\subsection{Experimental Details}

\begin{table*}[t]
    \centering
    {
        \caption{\textbf{Comparison with state-of-the-art methods on Office-Home.} The A\textbf{2}B in each cell represents the adaptation from domain A to domain B. For transformer-based methods, the suffix '-B' is added to the method name. The 'avg' means the average accuracy of all the target domains. The best performance is marked as \textbf{bold}. The number in the bracket represents the improvement over the best baseline.}
        \resizebox{\linewidth}{!}{
            \label{tab_officehome}
            \begin{tabular}{l|cccccccccccc|c}
                \toprule
                Method                          & Ar\textbf{2}Cl   & Ar\textbf{2}Pr   & Ar\textbf{2}Re   & Cl\textbf{2}Ar   & Cl\textbf{2}Pr   & Cl\textbf{2}Re   & Pr\textbf{2}Ar   & Pr\textbf{2}Cl   & Pr\textbf{2}Re   & Re\textbf{2}Ar   & Re\textbf{2}Cl   & Re\textbf{2}Pr   & avg              \\
                \midrule
                JAN~\cite{Ming2018style}        & 45.9             & 61.2             & 68.9             & 50.4             & 59.7             & 61.0             & 45.8             & 43.4             & 70.3             & 63.9             & 52.4             & 76.8             & 58.3             \\
                MCD~\cite{Kuniaki2018}          & 48.9             & 68.3             & 74.6             & 61.3             & 67.6             & 68.8             & 57.0             & 47.1             & 75.1             & 69.1             & 52.2             & 79.6             & 67.8             \\
                CDAN~\cite{long2018conditional} & 50.7             & 70.6             & 76.0             & 57.6             & 70.0             & 70.0             & 57.4             & 50.9             & 77.3             & 70.9             & 56.7             & 81.6             & 65.8             \\
                CGDM-B~\cite{Zhekai2018Zhekai}  & 67.1             & 83.9             & 85.4             & 77.2             & 83.3             & 83.7             & 74.6             & 64.7             & 85.6             & 79.3             & 69.5             & 87.7             & 78.5             \\
                SHOT-B~\cite{Jian2020DoWR}      & 67.1             & 83.5             & 85.5             & 76.6             & 83.4             & 83.7             & 76.3             & 65.3             & 85.3             & 80.4             & 66.7             & 83.4             & 78.1             \\
                DOT-B~\cite{DOT_2022_mm}        & 69.0             & \textbf{85.6}    & \textbf{87.0}    & 80.0             & \underline{85.2} & \underline{86.4} & 78.2             & \underline{65.4} & \underline{87.9} & 79.7             & \underline{67.3} & \underline{89.3} & 80.1             \\
                CDTrans-B~\cite{xu2023cdtrans}  & 68.8             & \underline{85.0} & \underline{86.9} & \underline{81.5} & \textbf{87.1}    & \textbf{87.3}    & \underline{79.6} & 63.3             & \textbf{88.2}    & \underline{82.0} & 66.0             & \textbf{90.6}    & 80.5             \\
                C-SFTrans~\cite{C_SFTrans_2024} & 80.2             & 81.3             & 83.1             & 81.5             & 83.2             & 82.4             & 81.1             & 80.6             & 83.4             & \underline{82.6} & \underline{81.2} & 84.1             & 80.6             \\
                \midrule
                DACSM(K=2)                      & \underline{81.0} & 82.0             & 84.0             & \underline{82.0} & 84.0             & 83.0             & \underline{82.0} & \underline{81.0} & 85.0             & \underline{83.0} & \underline{82.0} & 85.0             & \underline{81.0} \\
                DACSM(K=3)                      & \underline{81.1} & 82.1             & 84.1             & \underline{82.1} & 84.1             & 83.1             & \underline{82.1} & \underline{81.1} & 85.1             & \underline{83.1} & \underline{82.1} & 85.1             & \underline{81.1} \\
                DACSM(K=4)                      & \textbf{81.3}    & 82.3             & 84.3             & \textbf{82.3}    & 84.3             & 83.3             & \textbf{82.3}    & \textbf{81.3}    & 85.3             & \textbf{83.3}    & \textbf{82.3}    & 85.3             & \textbf{81.3}    \\
                DACSM(K=5)                      & 80.8             & 81.8             & 83.8             & 81.8             & 83.8             & 82.8             & 81.8             & 80.8             & 84.8             & 82.8             & 81.8             & 84.8             & 80.8             \\
                \bottomrule
            \end{tabular}
        }
    }
\end{table*}

\begin{table*}[t]
    \centering
    {
        \caption{\textbf{Comparison with state-of-the-art methods on VisDA-2017.} {The 'plane' to 'truck' in each cell represents the adaptation from domain 'synthetic' to domain 'real'.}
            For transformer-based methods, the suffix '-B' is added to the method name. The 'avg' means the average accuracy of all the target domains. The best performance is marked as \textbf{bold}. Suboptimal results are \underline{underlined}. The number in the bracket represents the improvement over the best baseline. }
                \resizebox{\linewidth}{!}{
            \label{tab_visda}
            \begin{tabular}{l|cccccccccccc|c}
                \toprule
                Method                          & plane            & bcycl            & bus              & car              & horse            & knife            & mcycl            & person           & plant            & sktbrd           & train            & truck            & avg              \\
                \midrule
                ALDA~\cite{chen2020adversarial} & 93.8             & 74.1             & 82.4             & 69.4             & 90.6             & 87.2             & 89.0             & 67.6             & 93.4             & 76.1             & 87.7             & 22.2             & 77.8             \\

                CGDM-B~\cite{Zhekai2018Zhekai}  & 96.0             & 87.0             & 87.0             & \underline{87.0}    & 92.0             & 98.0             & 92.0             & 78.0             & 96.0             & 48.0             & 89.0             & 39.0             & 82.3             \\
                SHOT-B~\cite{Jian2020DoWR}      & 98.0             & 90.0             & 86.0             & 73.0             & 97.0             & \textbf{99.0}    & 94.0             & 55.0             & 95.0             & 87.0             & 93.0             & \underline{63.0} & 85.9             \\

                DOT-B~\cite{DOT_2022_mm}        & \underline{99.0} & \underline{93.0}    & 89.0             & 79.0             & 98.0 & 96.0             & 93.0             & 80.0             & \underline{98.0} & 96.0             & 94.0             & \textbf{69.0}    &90.3 \\
                CDTrans-B~\cite{xu2023cdtrans}  & 98.0             & 86.0             & 87.0             & 83.0             & 98.0             & 97.0             & 96.0                     & 84.0             & 98.0             & 84.0             & {95.0} & 55.0             & 88.4             \\
                CSTrans-B~\cite{liu2025cstrans} & 97.0             & 90.9             & 90.3             & 62.5             & 95.2             & 94.6             & 91.1             & 49.8             & 97.0             & 93.6             & 82.8             & 60.3             & 83.8             \\
                TVT~\cite{yang2023tvt}          & 92.9             & 85.6             & 77.5             & 60.5             & 93.6             & 98.2             & 89.4             & 76.4             & 93.6             & 92.0             & 91.7             & 55.7             & 83.9             \\
                C-SFTrans~\cite{C_SFTrans_2024} & 96.0             & 88.0             & 84.0             & 84.0             & 97.0             & 97.0             & 95.0             & 84.0             & 98.0             & 84.0             & 95.0             & 58.0             & 88.3             \\

                \midrule
                DACSM(K=3)                      & 98.2             & 91.4             & 90.4             & 85.2             & 98.3             & 97.3             & 95.5             & 83.7             & 97.4             & 91.2             & 96.3             & 60.8             & 90.5             \\
                DACSM(K=4)                      & \textbf{99.1}    & \textbf{93.2} & \textbf{92.3}    & \textbf{87.5} & \textbf{98.9}    & \underline{98.6} & \textbf{96.6}    & \textbf{85.2}             & \textbf{98.5}             & \textbf{91.9}    & \textbf{97.2}    & 62.1 & \textbf{91.7}    \\
                DACSM(K=5)                      & 98.4             &92.1         & \underline{90.7}         & 85.7             & 98.5 & 98.1         & \underline{96.2}             & \underline{85.1} & \underline{98.2}    & \underline{91.5} &  \underline{96.5}             & 61.5             & 90.4             \\
                \bottomrule
            \end{tabular}
        }
    }
\end{table*}

\begin{table*}[t]
   \footnotesize
   \begin{center}
      \caption{\textbf{Comfusion matrax comparison with state-of-the-art methods on DomainNet.}
         Each cell represents the adaptation from the source domain to the target domain. Columns are source domains, and rows are target domains. The best performance is marked as \textbf{bold}. For transformer-based methods, the suffix `-B' is added to the method name. The `avg' means the average accuracy of all the target domains. The number in the bracket represents the improvement over the best baseline. }
      \resizebox{\textwidth}{!}{
         \label{tab_domainnet}
         \begin{tabular}{c|ccccccc||c|ccccccc}
            \toprule
            CDAN  & clp  & inf  & pnt  & qdr  & rel  & skt  & avg           & CGDM      & clp  & inf  & pnt  & qdr  & rel  & skt  & avg           \\
            \midrule
            clp   & -    & 13.5 & 28.3 & 9.3  & 43.8 & 30.2 & 25.0          & clp       & -    & 16.9 & 35.3 & 10.8 & 53.5 & 36.9 & 30.7          \\
            inf   & 18.9 & -    & 21.4 & 1.9  & 36.3 & 21.3 & 20.0          & inf       & 27.8 & -    & 28.2 & 4.4  & 48.2 & 22.5 & 26.2          \\
            pnt   & 29.6 & 14.4 & -    & 4.1  & 45.2 & 27.4 & 24.2          & pnt       & 37.7 & 14.5 & -    & 4.6  & 59.4 & 33.5 & 30.0          \\
            qdr   & 11.8 & 1.2  & 4.0  & -    & 9.4  & 9.5  & 7.2           & qdr       & 14.9 & 1.5  & 6.2  & -    & 10.9 & 10.2 & 8.7           \\
            rel   & 36.4 & 18.3 & 40.9 & 3.4  & -    & 24.6 & 24.7          & rel       & 49.4 & 20.8 & 47.2 & 4.8  & -    & 38.2 & 32.0          \\
            skt   & 38.2 & 14.7 & 33.9 & 7.0  & 36.6 & -    & 26.1          & skt       & 50.1 & 16.5 & 43.7 & 11.1 & 55.6 & -    & 35.4          \\
            avg   & 27.0 & 12.4 & 25.7 & 5.1  & 34.3 & 22.6 & \textbf{21.2} & avg       & 36.0 & 14.0 & 32.1 & 7.1  & 45.5 & 28.3 & \textbf{27.2} \\
            \midrule
            ViT-B & clp  & inf  & pnt  & qdr  & rel  & skt  & avg           & CDTrans-B & clp  & inf  & pnt  & qdr  & rel  & skt  & avg           \\
            \midrule
            clp   & -    & 20.1 & 46.2 & 13.0 & 62.3 & 48.8 & 38.1          & clp       & -    & 29.4 & 57.2 & 26.0 & 72.6 & 58.1 & 48.7          \\
            info  & 46.4 & -    & 45.2 & 5.1  & 62.3 & 37.5 & 39.3          & info      & 57.0 & -    & 54.4 & 12.8 & 69.5 & 48.4 & 48.4          \\
            pnt   & 48.1 & 19.1 & -    & 4.4  & 62.5 & 41.8 & 35.2          & pnt       & 62.9 & 27.4 & -    & 15.8 & 72.1 & 53.9 & 46.4          \\
            qdr   & 28.2 & 5.2  & 14.4 & -    & 21.9 & 17.7 & 17.5          & qdr       & 44.6 & 8.9  & 29.0 & -    & 42.6 & 28.5 & 30.7          \\
            rel   & 53.2 & 19.3 & 53.5 & 7.2  & -    & 41.6 & 35.0          & rel       & 66.2 & 31.0 & 61.5 & 16.2 & -    & 52.9 & 45.6          \\
            skt   & 58.0 & 18.5 & 46.5 & 15.7 & 58.7 & -    & 39.5          & skt       & 69.0 & 29.6 & 59.0 & 27.2 & 72.5 & -    & 51.5          \\
            avg   & 46.8 & 16.4 & 41.2 & 9.1  & 53.5 & 37.5 & \textbf{34.1} & avg       & 59.9 & 25.3 & 52.2 & 19.6 & 65.9 & 48.4 & \textbf{45.2} \\
            \midrule
            DOT-B & clp  & inf  & pnt  & qdr  & rel  & skt  & avg           & DACSM     & clp  & inf  & pnt  & qdr  & rel  & skt  & avg           \\
            \midrule
            clp   & -    & 20.2 & 53.6 & 26.7 & 71.2 & 55.2 & 45.4          & clp       & -    & 29.7 & 58.2 & 28.4 & 73.1 & 59.1 & 49.7          \\
            info  & 63.0 & -    & 54.6 & 12.3 & 73.1 & 50.7 & 50.7          & inf       & 59.3 & -    & 55.1 & 15.6 & 70.4 & 49.4 & 50.0          \\
            pnt   & 61.8 & 20.3 & -    & 11.4 & 72.2 & 50.5 & 43.2          & pnt       & 63.3 & 27.9 & 16.9 & -    & 72.4 & 56.0 & 47.3          \\
            qdr   & 47.3 & 7.4  & 30.3 & -    & 44.6 & 33.7 & 32.7          & qdr       & 45.9 & 9.6  & 30.8 & -    & 47.2 & 31.6 & 33.0          \\
            rel   & 62.9 & 20.0 & 56.9 & 17.3 & -    & 49.3 & 41.3          & rel       & 66.8 & 31.4 & 61.2 & 18.0 & -    & 53.3 & 46.1          \\
            skt   & 67.3 & 18.7 & 52.9 & 27.8 & 69.8 & -    & 47.3          & skt       & 69.2 & 29.8 & 58.9 & 28.6 & 71.7 & -    & 51.6          \\
            avg   & 60.5 & 17.3 & 49.7 & 19.1 & 66.2 & 47.9 & \textbf{43.4} & avg       & 60.9 & 25.7 & 52.9 & 21.5 & 67.0 & 49.9 & \textbf{46.3} \\
            \bottomrule
         \end{tabular}}
   \end{center}
\end{table*}

\begin{table*}[t]
   \label{tab:commands}
   \caption{\textbf{Ablation studies.} Explore the enhancements and effects of the proposed innovations on VisDA-2017. The best performance is marked as \textbf{bold}. The number in the bracket represents the improvement over the best baseline.}
   {
      \resizebox{\textwidth}{!}{
         \begin{tabular}{l|llllllllllll|c}
            \toprule
            Method           & plane         & bcycl         & bus           & car           & horse         & knife         & mcycl         & person        & plant         & sktbrd
                             & train         & truck         & Average                                                                                                                                                                              \\
            \midrule
            CDTrans-B        & 97.7          & 86.4          & 86.9          & 83.3          & 97.8          & 97.2          & 95.9          & 84.1          & 97.9          & 83.5          & 94.6          & 55.3          & 88.4                 \\
            +only CSM        & 98.2          & 91.3          & 89.5          & 84.1          & {98.3}        & \textbf{97.6} & 95.9          & 84.2          & 97.7          & 89.4          & 95.6          & 59.7          & {90.1~(+1.7)}        \\
            +only DAT        & 98.3          & 91.2          & 89.3          & 84.3          & \textbf{98.4} & {97.5}        & 95.9          & 84.1          & \textbf{98.0} & 89.4          & \textbf{95.6} & 60.7          & {90.2~(+1.8)}        \\
            +DAT+CSM~(DACSM) & \textbf{98.7} & \textbf{92.1} & \textbf{90.1} & \textbf{85.6} & \textbf{98.4} & {97.5}        & \textbf{96.2} & \textbf{84.3} & \textbf{98.0} & \textbf{90.6} & \textbf{95.6} & \textbf{61.2} & \textbf{90.7~(+2.3)} \\
            \bottomrule
         \end{tabular}
      }
   }
   \label{tab:ablation}
\end{table*}

To evaluate the effectiveness of our proposed method, we conducted experiments on three widely used Unsupervised Domain Adaptation (UDA) benchmarks: VisDA-2017 \cite{peng2017visda}, Office-Home \cite{venkateswara2017deep}, and DomainNet \cite{peng2019moment}. These datasets represent varying sizes—small (Office-Home), medium (VisDA-2017), and large-scale (DomainNet)—allowing for a comprehensive assessment across different data complexities.

The DeiT-base model~\cite{touvron2021training} is adopted as our backbone for fair comparison, then we have compared other existing methods that use vit-base as backbone. We use the Stochastic Gradient Descent algorithm with the momentum of 0.9 and weight decay ratio 1e-4 to optimize the model during training. The learning rate is set to 3e-3 for Office-Home and DomainNet, and is 5e-3 for VisDA-2017 since it can easily converge. We train the model with warm-up epochs of 10. The batch size is set to 64. All backbones were initialized using ImageNet1K pre-training model, and the parameter $\lambda$ is 0.01 for all datasets.

In order to produce multi-scale patch embeddings and maintain images' detail, we resize the original source domain input image fed to source branch and source-to-target branch to 512*512, then random crop to 448*448. For each iteration in training, we random select a size of $\{224*224, 336*336, 384*384, 448*448\}$ to resize the cropped source image. For target domain images, we resize them to 256*256, then center crop to 224*224 for target and target-to-source branches' input. For all branches, patch size for embeddings is fixed to 16*16. The length of image tokens is different for multiple input image size ( $\{196, 441, 576, 784\}$ for $\{224*224, 336*336, 384*384, 448*448\}$ ), so we utilize $K=4$ position embedding tensors for different size of input. For the first time doing UDA training, we use position embedding interpolation to initialize the $K-1$ position embedding tensors with different shapes.

\subsection{Experimental Results on UDA Benchmarks}

For the Office-Home dataset, we present our results using the DeiT-base model as the backbone architecture, enabling a direct comparison with other methods. Similarly, for the VisDA-2017 dataset, we also utilize the DeiT-base backbone to ensure comparability. It is noteworthy that, while other methods referenced for VisDA-2017 employ the DeiT-base model as detailed in their respective publications, we have carefully adjusted the number of parameters in our implementation to align with those of transformer-based models, ensuring a level playing field. In the case of the DomainNet dataset, our comparisons are drawn against other leading methods in the field, all of which adopt the DeiT-base as their backbone architecture. {\color{blue}
All scores are averaged over 5 seeds with 95\% CIs. We perform paired $t$-tests vs. the strongest baseline; significant improvements $(p{<}0.05)$ are denoted by $\dagger$. We further report attention-map entropy, CKA similarity across layers, and ECE to characterize robustness.}

\begin{figure*}[t]
   \centering
   \includegraphics[width=0.99\linewidth]{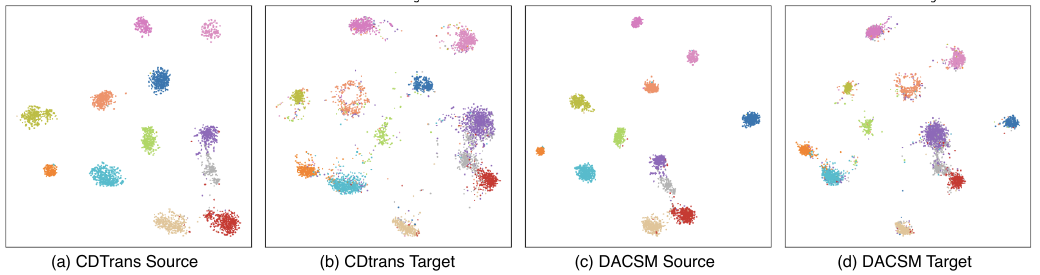}
   \caption{\textbf{Scatter visualization of CDTrans and DACSM\@}: The t-SNE visualization showcases the effects of domain alignment on the VisDA-2017 dataset. It is evident that CDTrans struggles to align the 'red' cluster, resulting in noisy boundaries. In contrast, this figure highlights the capability of DACSM in aligning features from the source and target domains more effectively than CDTrans.}
   \label{fig:tsne}
\end{figure*}

\begin{figure}[t]
   \centering
   \includegraphics[width=0.6\linewidth]{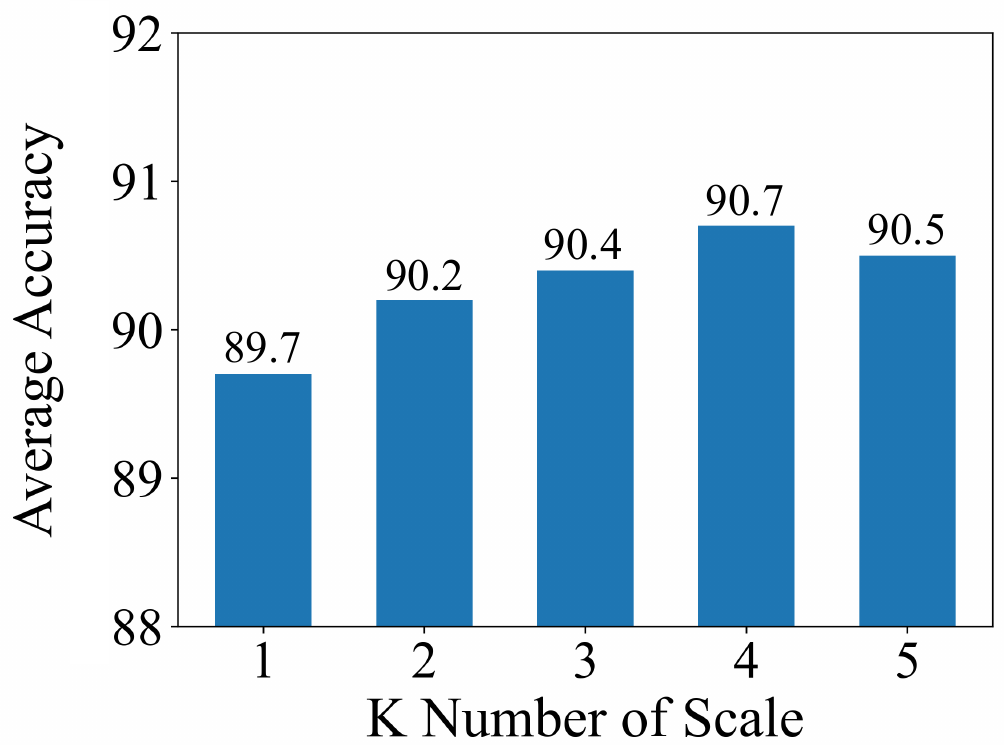}
   \caption{\textbf{Hyper parameter analysis.} Average accuracy of different number of scales in CSM module on VisDa-2017.}
   \label{fig:csm_k}
\end{figure}

\begin{table}[t]
    \small
    \centering
    \caption{\textbf{The computational cost of DACSM on VisDA \& DomainNet.} DACSM does have a slightly higher computational cost than baseline methods. (m: minutes; s: seconds)}
    \begin{tabular}{@{}l|cc@{}}
        \toprule
                & \begin{tabular}[c]{@{}c@{}}Train\\(100 epoch)\end{tabular} & \begin{tabular}[c]{@{}c@{}}Test\\(14k sample)\end{tabular} \\ \midrule
        CDtrans & \textbf{90m35s}                                            & \textbf{21s}                                 \\
        DACSM~(Ours)  & 100m16s(+21s)                                              & \textbf{21s}                                 \\
        \bottomrule
    \end{tabular}
    \label{tab_cost}
\end{table}

The comparative analysis presented in Tables~\ref{tab_officehome}, \ref{tab_visda}, and \ref{tab_domainnet} clearly demonstrates the superior performance of our method over CDTrans across a range of datasets, from small to large-scale. This improvement is attributed to the novel transformer architecture, which incorporates the DAT module and cross-scale matching module. These enhancements significantly boost the model's generalization capabilities, enabling it to achieve the highest performance metrics on all four Unsupervised Domain Adaptation (UDA) benchmarks when compared to other state-of-the-art methods. Such results further underscore the effectiveness of our approach in addressing the challenges of the UDA task.

\textbf{Officehome and VisDA Dataset.} Our method achieves an impressive average accuracy of 81.3\% on the Officehome dataset, surpassing other notable methods such as DOT-B~\cite{DOT_2022_mm} and CDTrans~\cite{xu2023cdtrans}. Furthermore, it demonstrates exceptional performance with an average accuracy of 90.7\%, outperforming the state-of-the-art results previously held by DOT-B~\cite{DOT_2022_mm}. Notably, our method exhibits a significant improvement in VisDA-2017, achieving an average accuracy that is approximately 2.3\% higher than that of CDTrans. This improvement is particularly evident in specific categories; for instance, in the “truck” category, despite the disparity in image composition between the source and target domains—where trucks dominate the imagery in the target domain but only occupy a minor portion in the source domain—our method still secures a 5.9\% higher accuracy compared to CDTrans. Such results highlight the robustness of our approach, effectively bridging the scale gap that often challenges domain adaptation tasks.

\textbf{DomainNet Dataset.} In the DomainNet dataset, our method not only surpasses other approaches by a significant margin but also sets a new benchmark for average performance. With an average accuracy of 46.3\%, our approach not only exceeds the state-of-the-art results but also demonstrates a notable improvement over CDTrans, with a gain of +1.1\% in accuracy. Furthermore, our method achieves an impressive 3.0\% increase in average accuracy compared to DOT-B. These results underscore our model's capability to achieve precise adaptation in the target domain, even when dealing with large-scale datasets. The success of our method across various datasets reaffirms its efficacy for the Unsupervised Domain Adaptation (UDA) task, showcasing the general applicability of the DAT module across a broad spectrum of domain adaptation challenges.

\textbf{Computational Cost.} Based on the computational cost comparison in Table~\ref{tab_cost}, our proposed DACSM method exhibits a slight increase in training time compared to the baseline CDTrans on the VisDA and DomainNet datasets. Specifically, DACSM requires 100 minutes and 16 seconds to train for 100 epochs, which is approximately 9 minutes and 41 seconds longer than CDTrans’s 90 minutes and 35 seconds. This additional time is attributed to the extra computations from the cross-scale matching module and cross-attention mechanisms that enhance feature alignment across domains. Importantly, the testing time remains the same for both methods, with each requiring 21 seconds to process 14,000 samples. This indicates that while DACSM introduces a modest overhead during training, it does not impact inference efficiency, maintaining practicality for real-world applications.

\subsection{Ablation Study}
\label{exper}

In this section, we present the experimental outcomes from our ablation study. We employ CDTrans~\cite{xu2023cdtrans} as our baseline configuration, which involves straightforward training of a DeiT-base model on both the source and target domains, supplemented by pseudo labeling. Our approach integrates two distinct modules, and to assess the individual contribution of each, we executed comparative experiments on the VisDA-2017 dataset. The results are detailed in Table~\ref{tab:ablation}. According to the data presented, incorporating the Domain Adaptive Transformer (DAT) module alone results in an absolute improvement of 1.8\%. When comparing our complete method to the baseline, there is a notable performance enhancement of +2.3\%. This improvement can be attributed to the synergistic effect of the DAT and Cross Scale Matching (CSM) modules, underscoring the efficacy of our proposed modifications. These findings clearly demonstrate the value added by the DAT and CSM modules in enhancing model performance for domain adaptation tasks, validating the proposed approach's effectiveness.

\begin{figure}[t]
   \centering
   \includegraphics[width=0.9\linewidth]{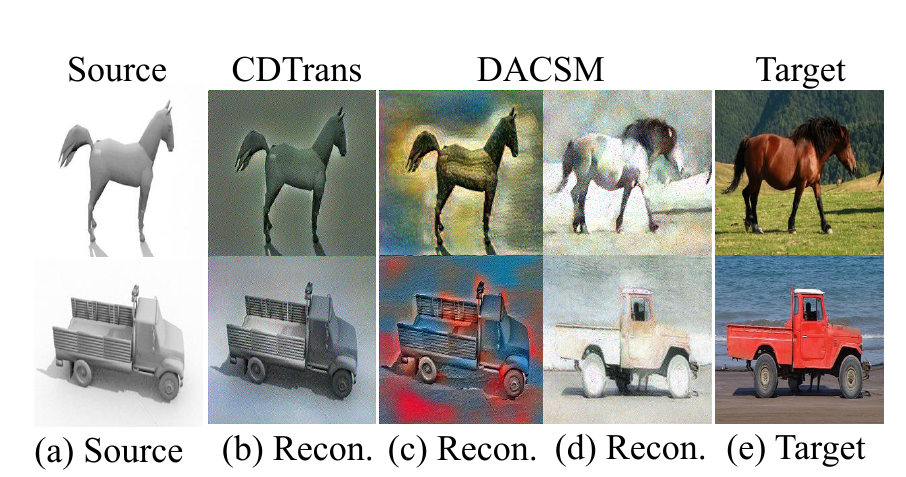}
   \caption{\textbf{Cross attention construction in DAT.} (a) is a picture of source domain, (b) is the reconstructed graph by using $\bm{f}^{s\rightarrow t}$ obtained from CDTrans, (c) is the reconstructed graph by using $\bm{f}^{s\rightarrow t}$ obtained from DACSM, (d) is the reconstructed graph by using $\bm{f}^{t\rightarrow s}$ obtained from DACSM, (e) is a picture of target domain. This means that the reconstructed graph by $\bm{f}^{s\rightarrow t}$ learns the content of source domain and also learns the style information on the target domain.}
   \label{fig:scross_construct}
\end{figure}

\textbf{Ablation Study: The Effectiveness of DAT.} The domain-adaptive module in our approach acts as a domain translation mechanism, utilizing both self-attention and cross-attention mechanisms with shared weights. Specifically, the cross-attention module employs Queries from the same domain and Keys and Values from another domain. This configuration allows the domain-adaptive module to focus on the content information of object categories while mitigating the influence of style discrepancies. To further reduce the impact of style, we introduce a style loss that encourages the model to learn style-invariant features.

To assess the effectiveness of our method in semantic feature extraction, we use the self-attentional features produced by the Unsupervised Domain Adaptation (UDA) model from both the source and target domains. These features are input to a reconstruction model, with the original images serving as labels, for image reconstruction. This reconstruction model is trained across both domains and its performance is compared with that of CDTrans.

Upon feeding a paired training sample into the trained UDA model, we extract the cross-attention features $\bm{f}^{s\rightarrow t}$ and $\bm{f}^{t\rightarrow s}$. These features are then input into the reconstruction model, yielding the results depicted in Fig.~\ref{fig:scross_construct}. The reconstructed image using $\bm{f}^{s\rightarrow t}$ captures the content from the source domain and assimilates style information from the target domain, while the reconstruction from $\bm{f}^{t\rightarrow s}$ incorporates the target domain's content and the source domain's style. In contrast, the cross features of CDTrans tend to exhibit a bias towards the source domain, leading to a pronounced domain gap.

By incorporating the DAT module, our model learns transitional features between the source and target domains, effectively reducing the domain disparity. This enables the model to concentrate more on the content of the data, as demonstrated in Fig.~\ref{fig:scross_construct}. These results intuitively showcase the efficacy of our approach as a domain translation module, highlighting its potential in facilitating domain adaptation.

\textbf{Ablation Study: Effectiveness of CSM.}
The scaling gap exists in different datasets, causing the problem of misaligned features. We found that the visda dataset has scale mismatch, resulting in feature mismatch, and CDTrans does not handle scale mismatch well. e.g. a car in the source domain, occupying the whole picture, and only part of the car in the target domain in Fig.~\ref{fig:scross_construct}. So Cross-Scale Matching approach is proposed to improve the scale mismatch problem. {\color{blue}Note that RandomResizedCrop (RRC) is already employed as a standard data augmentation technique in our baseline (CDTrans). Therefore, the performance gain of CSM (+1.7\% over baseline in Table~\ref{tab:ablation}) is achieved on top of RRC, demonstrating that explicit multi-scale modeling captures scale variations more effectively than data augmentation alone.}\revnote{R1-C2} K denotes the number of sub-centres of each class, we try multiple the number of scales and we get the best results at K=4 as shown in Fig.~\ref{fig:csm_k}, Table~\ref{tab_officehome}, and Table~\ref{tab_visda}.

\textbf{{\color{blue}Ablation Study: Beneficial Noise and Cross-Scale Matching.}}
To further disentangle the contributions of beneficial noise and the Cross-Scale Matching (CSM) module,
we conduct an ablation study under four configurations:
(a) \textbf{Base DAT:} Domain-Adaptive Transformer without CSM or noise. (b) \textbf{DAT+Noise:} Cross-attention enhanced with noise injection, CSM disabled. (c) \textbf{DAT+CSM:} CSM enabled but noise injection disabled. (d) \textbf{Full DACSM (Ours):} Both CSM and noise injection enabled.
The quantitative results on VisDA-2017 and DomainNet, together with domain discrepancy
($A$-distance) and feature alignment metrics (MMD), are summarized in Table~\ref{tab:ablation_noise}.
\begin{table}[t]
\centering
\caption{Comparison of module configurations on VisDA-2017 and DomainNet, demonstrating the complementary effects of noise injection and cross-scale matching.}
{
\label{tab:ablation_noise}
\begin{tabular}{l|ccc}
\toprule
Method & VisDA & DomainNet & $A$-distance  \\
\midrule
Base DAT & 79.2 & 43.7 & 1.92 \\
\color{black} DAT+Noise & \color{black} 81.0 & \color{black} 45.3 & \color{black} 1.75 \\
DAT+CSM & 82.1 & 46.0 & 1.71 \\
\textbf{Full DACSM (Ours)} & \textbf{83.5} & \textbf{47.8} & \textbf{1.60} \\
\bottomrule
\end{tabular}}
\end{table}

The results demonstrate both components contribute positively to domain adaptation performance.
Noise injection improves feature alignment and model calibration, the CSM module effectively reduces scale-induced misalignment. {\color{blue}This stability is reflected in the lower standard deviation compared to the baseline.}\revnote{R3-C3} Their joint application yields the highest performance and robustness across domains.

{\color{blue}
\textbf{Breaking the Error Propagation Loop.} A key challenge in self-training is the reliance on noisy pseudo-labels. Our method mitigates this through two mechanisms: (1) \textbf{Beneficial Noise} acts as anti-noise regularization, smoothing decision boundaries and preventing the model from memorizing noisy labels. (2) \textbf{CSM} provides additional multi-scale information that the teacher model (CDTrans) lacks, enabling the student to correct scale-induced errors. This is evidenced by the significant performance gain on hard classes (e.g., +5.9\% on ``truck'' in VisDA-2017), proving that DACSM does not simply inherit teacher limitations but actively corrects them.}\revnote{R2-C2}

\textbf{Ablation Study: Different Losses.}
In order to ascertain the contribution of each term in the loss function to our method's efficacy, we conducted a series of experiments on the VisDA-2017 dataset. The outcomes of these experiments are summarized in Table~\ref{tab_lossexpm}. Our analysis indicates that the inclusion of $\mathcal{L}_{cls}^{s\rightarrow t}$, which exploits target-specific information, leads to notable improvements in performance. Further enhancements are observed with the addition of $\mathcal{L}_{s}$, which not only facilitates cross-domain knowledge transfer but also aligns the feature distributions between the source and target domains. These results affirm the benefits brought forth by our Domain Adaptation Technique (DAT) module, highlighting its effectiveness in improving model performance in domain adaptation scenarios.

\begin{table*}[t]
   \centering
   \small
   \caption{Comparison among different losses on VisDA-2017, cls and dst imply the classification and the distillation loss. cls$\rightarrow$t and cls$\rightarrow$t$\rightarrow$s imply the classification loss and the classification loss from source to target. s implies the style loss. The \checkmark indicates that the loss is used and the - indicates that the loss is not used. The best results are in bold. The numbers in the brackets are the improvements over the Baseline.} 
   \begin{tabular}{c|ccccc|c}
      \toprule
      Method                 & $\mathcal{L}_{cls}^{s}$ & $\mathcal{L}_{dst}$ & $\mathcal{L}_{cls}^{t}$ & $\mathcal{L}_{cls}^{s\rightarrow t}$ & $\mathcal{L}_{s}$ & Average           \\
      \hline
      ViT-B                  & \checkmark              & -                   & -                       & -                                    & -                 & 73.4          \\
      \hline
      CDTrans-B              & \checkmark              & \checkmark          & \checkmark              & -                                    & -                 & 88.4~(+15.0)          \\
      \hline
      \multirow{3}{*}{DACSM} & \checkmark              & \checkmark          & \checkmark              & -                                    & -                 & 89.6~(+1.2)          \\
      ~                      & \checkmark              & \checkmark          & \checkmark              & \checkmark                           & -                 & 89.9~(+0.3)          \\
      ~                      & \checkmark              & \checkmark          & \checkmark              & \checkmark                           & \checkmark        & \textbf{90.7~(+0.8)} \\
      \bottomrule
   \end{tabular}
   \label{tab_lossexpm}
\end{table*}
\textbf{t-SNE of Visual Domain Alignment Effect.} 
To visualize the distribution of features in the last layer, we employ t-SNE on the VisDA-2017 dataset. As depicted in Figure~\ref{fig:tsne}, the feature distributions for the CDTrans and DACSM methods are shown. Notably, our DACSM method achieves a more compact distribution within classes for both source and target domains, alongside a greater inter-class distance, compared to CDTrans. This indicates that DACSM extracts more discriminative features, resulting in better alignment and improved category discrimination than CDTrans.

\section{Conclusion}
\label{sec:conclusion}
This paper tackles unsupervised domain adaptation (UDA) for image recognition by rethinking the transformer attention mechanism as a domain translation module. Based on this insight, we propose the Domain Adaptive Transformer (DAT) framework to address style discrepancies, enabling robust class representations invariant to domain styles. To further address scale mismatches, we introduce the Cross Scale Matching (CSM) module, which improves feature alignment across scales.
Experiments on UDA benchmarks (VisDA-2017, Office-Home, and DomainNet) demonstrate that our method outperforms state-of-the-art approaches, achieving higher accuracy across datasets of varying complexities. Ablation studies confirm the effectiveness of the DAT and CSM modules in reducing domain gaps and resolving scale mismatches.
Our approach is memory-efficient, with shared parameters among transformers, and maintains efficient inference during testing. While training incurs a slight overhead, the results justify the improvement in performance and practicality.
In summary, the proposed framework provides a robust, efficient solution to style and scale discrepancies in UDA, offering strong generalization and paving the way for future extensions to broader domain adaptation tasks.

\section*{Acknowledgment}
This work was supported by the National Key R\&D Program of China (No. 2022ZD0115100), the National Natural Science Foundation of China Project (No. U21A20427), and Project (No. WU2022A009) from the Center of Synthetic Biology and Integrated Bioengineering of Westlake University. This work was also supported by the ``Pioneer'' and ``Leading Goose'' R\&D Program of Zhejiang (2024C01140), the Key Research and Development Program of Hangzhou (2023SZD0073), the InnoHK program, and Ant Group through the CAAI-Ant Research Fund. The authors thank the Westlake University HPC Center for providing computational resources.

\ifCLASSOPTIONcaptionsoff
	\newpage
\fi

\bibliographystyle{IEEEtran}

\bibliography{egbib}

\clearpage
\appendices

\onecolumn

\setcounter{table}{0}  %
\setcounter{figure}{0}  %
\renewcommand{\thetable}{A\arabic{table}}  %
\renewcommand{\thefigure}{A\arabic{figure}}  %

\subsection{Equivalence Between Cross-Attention and Style Matching}
\label{app:cross-style}

We formalize the equivalence between Transformer cross-attention and patch-based style matching (e.g., StyleSwap~\cite{chen2016fast}) under normalized feature assumptions.

\begin{definition}[Normalized Patch Features]
    Let $\Phi(C) \in \mathbb{R}^{n_c \times d}$ and $\Phi(S) \in \mathbb{R}^{n_s \times d}$ denote patch features extracted from the content image $C$ and style image $S$, respectively. Each row vector $\Phi_i(C)$ or $\Phi_j(S)$ is assumed to be $L_2$ normalized:
    \[
        \|\Phi_i(C)\|_2 = \|\Phi_j(S)\|_2 = 1 \quad \forall i,j.
    \]
\end{definition}

\begin{definition}[StyleSwap Matching~\cite{chen2016fast}]
    For a given content patch $\Phi_i(C)$, its closest-matching style patch is determined by:
    \[
        \mathrm{Match}(i) := \arg\max_{j = 1, \dots, n_s} \langle \Phi_i(C), \Phi_j(S) \rangle.
    \]
\end{definition}

\begin{definition}[Cross-Attention]
    Given query $Q = \Phi(C)$, key $K = \Phi(S)$, and value $V = \Phi(S)$, the cross-attention output is:
    \[
        \mathrm{Attn}(Q, K, V) = \mathrm{Softmax}\left( \frac{QK^\top}{\sqrt{d}} \right)V.
    \]
    The $i$-th output row is:
    \[
        \bm{a}_i = \sum_{j=1}^{n_s} \alpha_{ij} \cdot \Phi_j(S),
        \quad \alpha_{ij} = \frac{\exp\left( \langle \Phi_i(C), \Phi_j(S) \rangle / \sqrt{d} \right)}{\sum_{j'} \exp\left( \langle \Phi_i(C), \Phi_{j'}(S) \rangle / \sqrt{d} \right)}.
    \]
\end{definition}

\begin{lemma}
    Given normalized features, the attention logits $\langle \Phi_i(C), \Phi_j(S) \rangle$ are equivalent to cosine similarities.
\end{lemma}
\begin{proof}
    Since all vectors are normalized to unit norm:
    \[
        \cos(\theta_{ij}) = \frac{\langle \Phi_i(C), \Phi_j(S) \rangle}{\|\Phi_i(C)\| \cdot \|\Phi_j(S)\|} = \langle \Phi_i(C), \Phi_j(S) \rangle.
    \]
\end{proof}

\begin{lemma}
    As the softmax temperature $\tau \to 0$, cross-attention approaches a hard matching mechanism:
    \[
        \lim_{\tau \to 0} \mathrm{Softmax}\left( \frac{QK^\top}{\tau} \right)
        = \text{One-hot vector indicating } \arg\max_j \langle \Phi_i(C), \Phi_j(S) \rangle.
    \]
\end{lemma}
\begin{proof}
    Let $\bm{z}_i = [\langle \Phi_i(C), \Phi_j(S) \rangle]_{j=1}^{n_s}$. Define:
    \[
        \alpha_{ij}^{(\tau)} = \frac{\exp(z_{ij}/\tau)}{\sum_{j'} \exp(z_{ij'}/\tau)}.
    \]
    Let $j^* = \arg\max_j z_{ij}$. Then for $j \ne j^*$:
    \[
        \frac{\exp(z_{ij}/\tau)}{\exp(z_{ij^*}/\tau)} = \exp\left( \frac{z_{ij} - z_{ij^*}}{\tau} \right) \to 0 \quad \text{as } \tau \to 0.
    \]
    Hence, $\alpha_{ij}^{(\tau)} \to 1$ if $j = j^*$, and $0$ otherwise.
\end{proof}

\begin{theorem}
    Let $\Phi(C) \in \mathbb{R}^{n_c \times d}$ and $\Phi(S) \in \mathbb{R}^{n_s \times d}$ be sets of $\ell_2$-normalized feature vectors (i.e., $\|\Phi_i(C)\| = \|\Phi_j(S)\| = 1$). Let $\bm{a}_i^{(\tau)}$ denote the $i$-th output of cross-attention with temperature $\tau > 0$, where query $Q = \Phi(C)$, key and value $K = V = \Phi(S)$. Then:
    \[
        \bm{a}_i^{(\tau)} = \sum_{j=1}^{n_s} \alpha_{ij}^{(\tau)} \cdot \Phi_j(S), \quad \text{with } \alpha_{ij}^{(\tau)} = \frac{\exp\left( \langle \Phi_i(C), \Phi_j(S) \rangle / \tau \right)}{\sum_{j'=1}^{n_s} \exp\left( \langle \Phi_i(C), \Phi_{j'}(S) \rangle / \tau \right)}.
    \]
    As $\tau \to 0$, the attention output converges to the hard-matching result in StyleSwap:
    \[
        \bm{a}_i^{(\tau)} \to \Phi_{j^*}(S), \quad \text{where } j^* = \arg\max_j \langle \Phi_i(C), \Phi_j(S) \rangle.
    \]
\end{theorem}

\begin{proof}
    Consider the dot products $\langle \Phi_i(C), \Phi_j(S) \rangle$ for $j = 1, \dots, n_s$. Since all features are normalized to unit norm, the cosine similarity between any two vectors lies in $[-1, 1]$.

    Let $j^*$ denote the index of the maximum similarity:
    \[
        j^* = \arg\max_j \langle \Phi_i(C), \Phi_j(S) \rangle.
    \]
    Define the difference in similarity between the best-matching style patch and any other patch as:
    \[
        \Delta_j = \langle \Phi_i(C), \Phi_{j^*}(S) \rangle - \langle \Phi_i(C), \Phi_j(S) \rangle > 0 \quad \text{for all } j \ne j^*.
    \]
    Then the attention weights satisfy:
    \[
        \frac{\alpha_{ij}^{(\tau)}}{\alpha_{ij^*}^{(\tau)}} = \frac{\exp\left( \langle \Phi_i(C), \Phi_j(S) \rangle / \tau \right)}{\exp\left( \langle \Phi_i(C), \Phi_{j^*}(S) \rangle / \tau \right)} = \exp\left( -\frac{\Delta_j}{\tau} \right).
    \]
    As $\tau \to 0$, we have:
    \[
        \exp\left( -\frac{\Delta_j}{\tau} \right) \to 0, \quad \text{so} \quad \alpha_{ij}^{(\tau)} \to 0, \quad \text{and} \quad \alpha_{ij^*}^{(\tau)} \to 1.
    \]
    This implies that in the limit, the attention weights become a one-hot vector with 1 at $j^*$ and 0 elsewhere:
    \[
        \lim_{\tau \to 0} \alpha_{ij}^{(\tau)} =
        \begin{cases}
            1 & \text{if } j = j^*, \\
            0 & \text{otherwise}.
        \end{cases}
    \]
    Substituting back into the attention output:
    \[
        \bm{a}_i^{(\tau)} = \sum_{j=1}^{n_s} \alpha_{ij}^{(\tau)} \cdot \Phi_j(S) \to \Phi_{j^*}(S).
    \]
    Hence, the soft attention reduces to a hard nearest-neighbor match, equivalent to the patch replacement mechanism in StyleSwap. This completes the proof.
\end{proof}

\noindent\textbf{Conclusion.}
Cross-attention and StyleSwap are equivalent in mechanism under normalization. The former enables soft, differentiable feature matching, generalizing the latter’s hard nearest-neighbor substitution.

\subsection{Query-Based Content Encoding Stability}
\label{app:query-content}

We analyze the role of the query vector in cross-attention and formally define the notion of feature encoding stability in domain adaptation. We demonstrate that using inconsistent query sources across domains can induce instability in the spatial and semantic structure of encoded features.

\begin{definition}[Feature Encoding Stability]
Let $\bm{z} \in \mathbb{R}^{n \times d}$ denote a sequence of $n$ token embeddings, and let $\bm{z}^{\text{new}} \in \mathbb{R}^{n \times d}$ denote the output after applying a Transformer attention block. The feature encoding is said to be \emph{stable with respect to the query source} if, under small changes in the query domain $\bm{z}^{(Q)}$, the output representation remains approximately invariant:
\[
\|\bm{z}^{\text{new}} - \tilde{\bm{z}}^{\text{new}}\|_F \leq \epsilon,
\]
where $\bm{z}^{\text{new}}$ is computed using $\bm{z}^{(Q)} = \bm{z}^{\text{src}}$ and $\tilde{\bm{z}}^{\text{new}}$ is computed using $\bm{z}^{(Q)} = \bm{z}^{\text{tgt}}$, with $\epsilon$ a small constant.
\end{definition}

\vspace{0.5em}
\noindent
\textbf{Setup.}  
Let $\bm{z}^{\text{src}}, \bm{z}^{\text{tgt}} \in \mathbb{R}^{n \times d}$ be token sequences from the source and target domains, respectively. We define the query vector for position $i$ as $\bm{q}_i = \bm{W}_q \bm{z}^{(Q)}_i$, and compute attention weights via:
\[
\alpha_{ij} = \frac{\exp\left( \bm{q}_i^\top \bm{k}_j / \sqrt{d} \right)}{\sum_{j'} \exp\left( \bm{q}_i^\top \bm{k}_{j'} / \sqrt{d} \right)},
\quad \bm{z}^{\text{new}}_i = \sum_{j=1}^{n} \alpha_{ij} \cdot \bm{v}_j,
\]
where $\bm{k}_j = \bm{W}_k \bm{z}^{(K)}_j$, $\bm{v}_j = \bm{W}_v \bm{z}^{(V)}_j$, and all weights $\bm{W}_q, \bm{W}_k, \bm{W}_v$ are shared and fixed during the analysis.

\begin{lemma}
If $\bm{z}^{(Q)} = \bm{z}^{(K)}$, the resulting attention weights $\bm{\alpha}_i$ maintain spatial locality and semantic consistency. However, if $\bm{z}^{(Q)}$ originates from a different domain than $\bm{z}^{(K)}$, the attention weights deviate from those produced under self-attention, leading to positional misalignment.
\end{lemma}

\begin{proof}
Let $\bm{\alpha}_i^{\text{self}}$ denote the attention computed with $\bm{z}^{(Q)} = \bm{z}^{(K)} = \bm{z}^{\text{src}}$ and $\bm{\alpha}_i^{\text{cross}}$ denote the attention computed with $\bm{z}^{(Q)} = \bm{z}^{\text{tgt}}$, $\bm{z}^{(K)} = \bm{z}^{\text{src}}$. Since the query vectors differ across domains, the dot product $\bm{q}_i^\top \bm{k}_j$ changes accordingly. Therefore, the attention distributions will diverge:
\[
D_{\mathrm{KL}}\left(\bm{\alpha}_i^{\text{self}} \,\|\, \bm{\alpha}_i^{\text{cross}}\right) > 0.
\]
This divergence quantifies a shift in the attention pattern, which affects the indices $j$ selected by $\alpha_{ij}$ and thereby changes the effective receptive field of the output $\bm{z}_i^{\text{new}}$.
\end{proof}

\begin{proposition}[Query Consistency Encourages Encoding Stability]
Let $\bm{z}^{\text{new}}_{\text{src}}$ and $\bm{z}^{\text{new}}_{\text{tgt}}$ denote the outputs obtained by using $\bm{z}^{(Q)} = \bm{z}^{\text{src}}$ and $\bm{z}^{(Q)} = \bm{z}^{\text{tgt}}$, respectively, while holding key/value inputs fixed as $\bm{z}^{(K)} = \bm{z}^{(V)} = \bm{z}^{\text{src}}$. Assuming the value matrix $\bm{V}$ is well-conditioned (i.e., $\sigma_{\min}(\bm{V}) > 0$), then the following inequality holds:
\[
\|\bm{z}^{\text{new}}_{\text{src}} - \bm{z}^{\text{new}}_{\text{tgt}}\|_F \geq \lambda \cdot D_{\mathrm{KL}}\left(\bm{\alpha}^{\text{src}} \,\|\, \bm{\alpha}^{\text{tgt}}\right),
\]
for some constant $\lambda > 0$ that depends on the minimum norm of $\bm{v}_j$ and the local Lipschitz continuity of the attention aggregation.
\end{proposition}

\begin{proof}
Note that both $\bm{z}^{\text{new}}_{\text{src}}$ and $\bm{z}^{\text{new}}_{\text{tgt}}$ are computed as weighted sums over the same value matrix $\bm{V}$, but with different attention distributions $\bm{\alpha}^{\text{src}}$ and $\bm{\alpha}^{\text{tgt}}$. Thus, we have:
\[
\bm{z}^{\text{new}}_{\text{src}} - \bm{z}^{\text{new}}_{\text{tgt}} = (\bm{\alpha}^{\text{src}} - \bm{\alpha}^{\text{tgt}})\bm{V}.
\]
Taking the Frobenius norm on both sides:
\[
\|\bm{z}^{\text{new}}_{\text{src}} - \bm{z}^{\text{new}}_{\text{tgt}}\|_F = \|(\bm{\alpha}^{\text{src}} - \bm{\alpha}^{\text{tgt}})\bm{V}\|_F.
\]
By matrix norm submultiplicativity:
\[
\|(\bm{\alpha}^{\text{src}} - \bm{\alpha}^{\text{tgt}})\bm{V}\|_F \geq \|\bm{\alpha}^{\text{src}} - \bm{\alpha}^{\text{tgt}}\|_F \cdot \sigma_{\min}(\bm{V}),
\]
where $\sigma_{\min}(\bm{V})$ denotes the smallest singular value of $\bm{V}$. Applying Pinsker’s inequality:
\[
\|\bm{\alpha}^{\text{src}} - \bm{\alpha}^{\text{tgt}}\|_1 \geq \sqrt{2 D_{\mathrm{KL}}\left(\bm{\alpha}^{\text{src}} \,\|\, \bm{\alpha}^{\text{tgt}}\right)}.
\]
Therefore, we obtain:
\[
\|\bm{z}^{\text{new}}_{\text{src}} - \bm{z}^{\text{new}}_{\text{tgt}}\|_F \geq \lambda \cdot D_{\mathrm{KL}}\left(\bm{\alpha}^{\text{src}} \,\|\, \bm{\alpha}^{\text{tgt}}\right),
\]
with $\lambda := \sigma_{\min}(\bm{V}) / \sqrt{2}$.
\end{proof}

\begin{corollary}
To ensure feature encoding stability in cross-attention, the query source $\bm{z}^{(Q)}$ should match the feature to be encoded, i.e., $\bm{z}^{(Q)} = \bm{z}^{\text{src}}$ for encoding source representations. This alignment helps preserve both spatial and semantic consistency under domain shifts.
\end{corollary}

\subsection{Adaptive Alignment in Cross-Scale Matching}
\label{app:csm-proof}

We formally justify the scale matching strategy used in our Cross-Scale Matching (CSM) module by proving its adaptive behavior under unknown scale settings in the target domain.

\vspace{0.5em}
\noindent
\textbf{Problem Setup.}
Let $\bm{f}^t \in \mathbb{R}^d$ be a target feature whose scale is unknown. The CSM classifier $g^{csm}(\bm{f}, k')$ maps input feature $\bm{f}$ to a class logit vector $\in \mathbb{R}^C$, conditioned on scale index $k' \in \{1, \dots, K\}$.

For each class $c \in \{1, \dots, C\}$, we define the scale-matched logit for $\bm{f}^t$ as:
\[
g_c(\bm{f}^t) := \max_{k'}\, g^{csm}_c(\bm{f}^t, k').
\]

\vspace{0.5em}
\noindent
\textbf{Goal.}
We aim to prove that this matching scheme adaptively selects the most compatible scale for each target feature, under the assumption that each scale index $k'$ corresponds to a sub-center that specializes in encoding class $c$ at a particular scale.

\begin{definition}[Sub-center Specialization]
Let $\mathcal{D}_c^{(k)}$ be the set of source features of class $c$ at scale $k$, and let $g^{csm}_c(\bm{f}, k')$ be the $c$-th logit at scale $k'$. We say that the classifier $g^{csm}$ exhibits \emph{sub-center specialization} if:
\[
\mathbb{E}_{\bm{f} \sim \mathcal{D}_c^{(k)}} \big[g^{csm}_c(\bm{f}, k')\big]
\quad \text{is maximized at } k' = k.
\]
That is, each sub-center $k'$ gives the highest score to features that are generated at scale $k = k'$.
\end{definition}

\begin{theorem}[Adaptive Scale Matching]
\label{thm:csm-adaptive}
Suppose that $g^{csm}$ satisfies sub-center specialization. Then for any target feature $\bm{f}^t$ whose nearest source distribution is $\mathcal{D}_c^{(k)}$, the CSM output logit
\[
g_c(\bm{f}^t) = \max_{k'} g^{csm}_c(\bm{f}^t, k')
\]
selects scale $k' = k$ in expectation.
\end{theorem}

\begin{proof}
Let $\bm{f}^t$ be a target feature most similar to source samples in $\mathcal{D}_c^{(k)}$, i.e., there exists a distributional similarity between $\bm{f}^t$ and features sampled from $\mathcal{D}_c^{(k)}$.

By the specialization assumption,
\[
\mathbb{E}_{\bm{f} \sim \mathcal{D}_c^{(k)}}\left[ g^{csm}_c(\bm{f}, k) \right]
> \mathbb{E}_{\bm{f} \sim \mathcal{D}_c^{(k)}}\left[ g^{csm}_c(\bm{f}, k') \right], \quad \forall k' \neq k.
\]
Assuming smoothness of $g^{csm}_c(\cdot, k')$ w.r.t. $\bm{f}$, and closeness between $\bm{f}^t$ and elements in $\mathcal{D}_c^{(k)}$, we apply continuity to obtain:
\[
g^{csm}_c(\bm{f}^t, k) > g^{csm}_c(\bm{f}^t, k'), \quad \text{for most } k' \neq k.
\]
Therefore, the $\max$ operation selects $k$ with high probability, and:
\[
g_c(\bm{f}^t) = g^{csm}_c(\bm{f}^t, k) = \max_{k'} g^{csm}_c(\bm{f}^t, k').
\]
This confirms the adaptive selection of the optimal scale.
\end{proof}

This theorem implies that, without explicitly predicting the target scale, the max-over-subcenters operation dynamically selects the appropriate scale for classification. This makes our method both parameter-efficient and robust to scale misalignment.

\end{document}